%% file: main.tex
\newcommand{\rootpath}{./}

\documentclass[10pt, journal]{IEEEtran}
\ifCLASSINFOpdf
\else
\fi
\input{head}

\begin{document}
%

\title{A Generalizable Model-and-Data Driven Approach for Open-Set RFF Authentication}
%
%
%

\author{Renjie Xie,~\IEEEmembership{Student Member,~IEEE},
        Wei Xu,~\IEEEmembership{Senior Member,~IEEE},
        Yanzhi Chen,
        Jiabao Yu, \\
        Aiqun Hu,~\IEEEmembership{Senior Member,~IEEE},
        Derrick Wing Kwan Ng,~\IEEEmembership{Fellow,~IEEE},\\ and
        A. Lee Swindlehurst,~\IEEEmembership{Fellow,~IEEE}
        %

\thanks{R. Xie, W. Xu, J. Yu, and A. Hu are with the National Mobile Communications Research Laboratory, Southeast University, Nanjing 210096, China (e-mail: renjie\_xie@seu.edu.cn, wxu@seu.edu.cn, yujiabao@seu.edu.cn, aqhu@seu.edu.cn).\emph{(Corresponding author: Wei Xu)}}%
\thanks{Y. Chen is with School of Informatics, The University of Edinburgh, EH8 9AB, United Kingdom (e-mail: s1788325@inf.ed.ac.uk).}%
\thanks{D. W. K. Ng is with the School of Electrical Engineering and Telecommunications, University of New South Wales, Sydney, NSW 2052, Australia (e-mail: w.k.ng@unsw.edu.au).}%
\thanks{A. L. Swindlehurst is with the Center for Pervasive Communications and Computing, University of California, Irvine, CA 92697-2625 USA (e-mail: swindle@uci.edu).}%
}%

\maketitle

\begin{abstract}
Radio-frequency fingerprints~(RFFs) are promising solutions for realizing low-cost physical layer authentication. 
Machine learning-based methods have been proposed for RFF extraction and discrimination. However, most existing methods are designed for the closed-set scenario where the set of devices is remains unchanged. These methods can not be generalized to the RFF discrimination of unknown devices. 
To enable the discrimination of RFF from both known and unknown devices, we propose a new end-to-end deep learning framework for extracting RFFs from raw received signals. The proposed framework comprises a novel preprocessing module, called neural synchronization~(NS), which incorporates the data-driven learning with signal processing priors as an inductive bias from communication-model based processing.
Compared to traditional carrier synchronization techniques, which are static, this module estimates offsets by two learnable deep neural networks jointly trained by the RFF extractor.  Additionally, a hypersphere representation is proposed to further improve the discrimination of RFF. Theoretical analysis shows that such a data-and-model framework can better optimize the mutual information between device identity and the RFF, which naturally leads to better performance. 
Experimental results verify that the proposed RFF significantly outperforms purely data-driven DNN-design and existing handcrafted RFF methods in terms of both discrimination and network generalizability. 
\end{abstract}


%
\IEEEpeerreviewmaketitle

\input{sections/01_introduction}

\input{sections/02_systemmodel}

\input{sections/03_RFF}

\input{sections/04_analysis}

\input{sections/05_experiments}

\input{sections/06_discussion}


%

\appendices

\input{sections/07_appendix}

%



\ifCLASSOPTIONcaptionsoff
  \newpage
\fi




\input{main.bbl}

%



\end{document}

%% file: sections/01_introduction.tex
\section{Introduction}

\IEEEPARstart{T}{he} widespread use of wireless devices has raised the issue of massive device authentication in wireless communication.
Conventional authentication based on cryptographic techniques~\cite{chen2016fully,iwamoto2017security} has significant difficulty in detecting compromised keys. Moreover, as the number of devices increases, key-based authentication at the higher layers suffers from excessive latency caused by heavy computation in the key management procedures~\cite{fang2018learning}.
To this end, physical layer authentication (PLA) has become an alternative solution for fast and efficient authentication of a large number of connected wireless devices.
Compared with conventional higher layer authentication schemes, PLA exploits inherent physical layer properties of the device hardware~\cite{danev2012physical} and enables authentication with low latency, low power consumption, and low computational overhead~\cite{hou2014physical}. As a result, PLA has attracted considerable attentions in the past few years~ \cite{peng2019deep,youssef2018machine,peng2018design,yu2019multi,patel2014improving,nguyen2011device,robyns2017physical,vo2016fingerprinting,merchant2018deep,huang2016specific,yu2019robust}.

Radio-frequency fingerprints~(RFFs) play an essential role in enabling PLA for device classification and authentication. In general, RFF refers to a set of physical layer features that are sufficient for uniquely identifying a wireless device. The quality of the RFFs crucially determines the reliability of PLA. Similar to human biometrics, these features are difficult to modify or tamper with~\cite{wang2016wireless}. 

Historically, RFF methods extracted physical layer features from the on/off transients of a radio signal received from a device. Methods of this kind date back to the work of Toonstra and Kinsner in 1996, where they distinguished seven VHF FM transmitters from four different manufacturers using wavelet analysis~\cite{toonstra1996radio}. \bflag{Later, other on/off transient features were also introduced for RFF, including phase offsets~\cite{knox2012practical}, amplitude, power, and DWT coefficients~\cite{hall2006detection,hall2004enhancing,hall2005radio}.}

\bflag{The above transient-based RFFs are sensitive to the position of the devices, the propagation environment, and the precision of the receivers~\cite{peng2018design}. 
To overcome these drawbacks, modulation-based methods were proposed to extract more stable features from a received signal~(e.g., the preamble)}. A representative work in~\cite{brik2008wireless} proposed to exploit a union of the synchronization correlation, in-phase/quadrature (I/Q) offset, phase offset, and magnitude of the received signal for RFF. Subsequent works further made use of the automatic gain control~(AGC) response~\cite{knox2010agc}, amplifier nonlinearity~\cite{huang2016specific}, sampling frequency offset~\cite{vo2016fingerprinting}, and carrier frequency offset~\cite{nguyen2011device,hou2014physical,robyns2017physical}, all of  which introduce various trade-offs between system complexity and authentication performance. The different proposed methodologies are "handcrafted", because they extract the RFF according to expert knowledge. Although they can be used in many situations, handcrafted RFF methods suffer from their inability to be generalized, and thus they are not suitable for general device classification.

\begin{figure*}[t]
	\centering
	\includegraphics[width=0.95\linewidth]{\rootpath/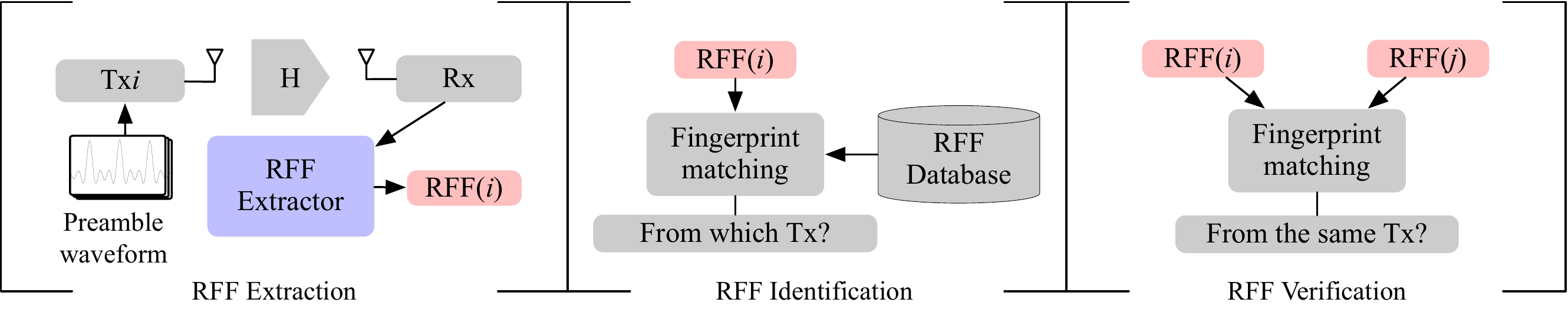}
	\caption{Open-set RFF-based physical layer authentication.}
	\label{fig:diagram_rff}
\end{figure*}

Machine learning techniques such as decision trees~\cite{patel2014improving}, linear classifiers~\cite{peng2018design}, $k$-nearest neighbor~($k$-NN)~\cite{tian2019new}, support vector machines~(SVM)~\cite{robyns2017physical,youssef2018machine}, etc., have been applied to the problem of RFF authentication. The effectiveness of these machine learning~(ML)-based RFF classifiers relies heavily on the quality of the extracted RFFs. In particular, if the extracted RFF features are not well-separable for different devices (e.g., if the relationship between the extracted RFF and the device identity is highly nonlinear), these traditional ML models (or ``shallow model'' in the taxonomy of the machine learning community) are usually incapable of correctly classifying devices using RFFs. Moreover, these traditional ML-based methods perform poorly in many real-world environments, as hardware imperfections contain nonlinear features that handcrafted RFF methods cannot easily model.

To overcome the limitations of the shallow models, deep neural networks~(DNN) have been employed in RFF authentication for better performance. In principle, DNNs are by design flexible models for representing nonlinear functions. They can either be seen as a more powerful classifier, or as a feature extraction framework that automatically learns high-level separable features~\cite{goodfellow2016deep} for simple classifiers. As a result, DNNs have achieved enormous success in a wide range of domains~\cite{krizhevsky2017imagenet, devlin2018bert, silver2017mastering, chen2020neural,rezende2015variational,gu2015neural, lu2020reinforcement}. For the field of RFF, the major advantage of adopting DNNs is their ability to exploit a wider family of RFFs, even if they are not directly separable for different devices. These non-directly separable features can potentially contain more information about the device identity. Under this assumption, the pioneering work in~\cite{merchant2018deep} extracted RFFs using a low-pass filter after applying synchronization on the raw signals, which facilitated the application of a convolutional neural network~(CNN) to extract important features. Thanks to the power of CNNs, a significant performance improvement over shallow model-based methods was observed. In~\cite{yu2019robust}, the authors further considered applying multiple sampling rates to the received signal for RFF extraction. Specifically, a CNN was trained to classify these sophisticated RFFs, and it effectively identified devices from highly non-separable RFFs.

It should be noted that despite the promising improvement, the performance of the above DL-based RFF methods still crucially depends on the quality of the original RFFs. However, typical preprocessing methods such as synchronization for extracting RFFs were designed for the general task of \emph{communication} rather than the desired task of discriminative \emph{RFF extraction}. These preprocessing procedures can weaken or even discard important information about the device identity. This loss of information can reduce the generalizability of the extracted RFF to the case of open set identification of unknown devices~\cite{geng2020recent}. This explains why most of the existing methods can only treat PLA as a closed-set identification problem where the set of devices remains static. It should be noted that the open-set setting is more realistic in real-world applications where the set of devices in the system often varies with time.
Directly feeding the raw signals to a neural network may not be viable as the trained network would presumably overfit the data. To preserve enough information about the device identifies while maintaining a balanced generalizability, a new preprocessing technique tailored for RFF extraction is needed.
To this end, this paper proposes a new RFF extraction framework that contains a special purpose preprocessing module for PLA. The main contributions of this paper are summarized as follows.

1) Methodologically, we propose an end-to-end RFF extraction framework that combines signal processing priors and deep learning for open-set RFF authentication. This combination is realized by a novel preprocessing module called neural synchronization~(NS), which generalizes traditional carrier synchronization (TS) with deep neural networks. Additionally, a hyperspherical representation is proposed to encourage large distances between devices in terms of a cosine distance metric of their RFFs. The resultant RFF can be directly used to distinguish either known or unknown devices, or even for outlier detection and new device discovery.

2) Theoretically, we prove that the learning process of the proposed RFF extraction framework is equivalent to optimizing a lower bound on the mutual information between the device identity and the RFF. In comparison, traditional handcrafted techniques designed for communication are unable to make such a claim to optimality. This observation highlights the necessity of the proposed learning framework in AI-assisted RFF authentication systems.

3) Experimentally, we demonstrate that the proposed framework outperforms state-of-the-art methods in terms of both robustness and accuracy for either closed-set or open-set RFF authentication tasks. This performance gain is mainly due to the use of inductive bias (synchronization) in the design of the resulting deep neural networks. We also verify the fact that traditional methods tend to rely on channel information rather than device information to distinguish devices, which explains their dissatisfactory performance in open-set settings. This conclusion sheds important light on the design of RFF extraction methods.

The rest of this paper is organized as follows. Section II describes the system model. Section III elaborates the details of the proposed method, while Section IV provides the theoretical analysis of the proposed method. Section V presents the experimental results. Finally, Section VI concludes this paper.

%% file: sections/02_systemmodel.tex
\section{System Model}

\subsection{RF Fingerprinting}
We consider a PLA system as shown in Fig. \ref{fig:diagram_rff}, containing $K$ transmitters~(Txs) denoted by $\{\text{Tx}_1, \text{Tx}_2, ..., \text{Tx}_K\}$ and one receiver~(Rx). The transmission of a preamble signal from $\text{Tx}_i, \forall i \in {1,2,..K}$ can be represented as 
\begin{equation}
	\vecr = f_i(\vecx),
\end{equation}
where $\vecx \in \mathbb{C}^{M}$ is the preamble signal with length $M$, $\vecr \in \mathbb{C}^{M}$ is the received signal at the Rx, $\mathbb{C}$ indicates the set of complex numbers, and the function $f_i$ represents the transmission between $\text{Tx}_i$ and $\text{Rx}$. Note that for an ideal channel, $f_i$ includes a representation of the hardware properties of $\text{Tx}_i$, therefore imprinting the identity-relevant information on the corresponding received signal $\vecr$. In reality, $f_i$ also contains distinctions related to the different channels among the devices that complicate the extraction of the hardware features.

As described in Fig.~\ref{fig:diagram_rff}, the goal of PLA is to determine the identity $\vecy$~(a $K$-dimensional one-hot vector) of a device from the received signal $\vecr$. This is typically done using a feature vector $\vecz$ (i.e., RFF) derived from the signals rather than using the signal itself. Let $\mathscr{F}\triangleq\{(\vecr^{(n)}, \vecy^{(n)})_{n=1}^N\}$ be the training set. A PLA can be formulated as a classification problem given by:
\begin{align}
	\label{eq:origin_proplem} &\underset{\W}{\operatorname{min}} \quad \mathcal{L}(\W) \triangleq 
	     - \frac{1}{N} \sum_{n=1}^{N} \ln p_{\W}(\vecy^{(n)}|\vecz^{(n)}) \\
	\label{eq:origin_condition} &\text{subject to} \quad\quad\quad\quad \vecz^{(n)}  =F(\vecr^{(n)}),
\end{align}
where $p_{\W}(\vecy|\vecz)$ is the probability that the system classifies $\vecz$ to device identity $\vecy$ (which is often realized by a softmax function), $\W=\{\{\vecw_j\}_{j=1}^{K}\}$ represents the parameter of the softmax function, and $F: \mathbb{C}^M \rightarrow\mathbb{R}^m$ is the function used to extract the RFF from the received signal, referred to as the RFF extractor. Different RFF solutions lead to different forms of $F(\cdot)$. In traditional RFF solutions,  $F(\cdot)$ is a hand-crafted signal processing function, whereas in recent deep learning-based approaches $F(\cdot)$ is replaced by a deep neural network\footnote{Strictly speaking, a composition of several preprocessing steps and a deep neural network.} to effectively mimic the non-linear signal processing.

\subsection{Open-Set Physical Layer Authentication}
Regardless of the specific form of the RFF extractor $F(\cdot)$, there is a common problem in prior RFF solutions, i.e., the inability to cope with \emph{unknown} devices. This is because the system is originally built to maximize the classification performance among already \emph{known} devices. In reality, however, it is impossible to obtain all potential devices in advance when training a classifier since the number of devices in the system is unlikely to remain unchanged --- a problem known as \emph{open-set} authentication~\cite{geng2020recent}. An obvious solution is to retrain the classifier whenever a new device enters the system ~\cite{hanna2020open,hanna2020deep}. While this approach is theoretically sound, it may incur huge system resources~(e.g., time, energy) in a practical deployment. As an alternative, we seek a practical, low-cost solution that avoids retraining as much as possible.

To this end, we propose a completely different solution to open-set RFF authentication. Compared with existing approaches that continually retrain the classifier, we learn and extract a discriminative RFF that can be generalized to unknown devices. In other words, the learned RFF extractor $F(\cdot)$  not only extracts distinguishing features for known devices but also discriminates against completely unknown devices as well. We then perform authentication by comparing the \emph{similarity} between RFFs using a given distance function (e.g., the Euclidean or cosine distance). If two RFFs, say $\vecz^{(i)}$ and $\vecz^{(j)}$, are highly similar to each other, they are considered to come from the same device; otherwise, they are assumed to be different. Mathematically, this procedure is formulated as follows:
\begin{equation}
\begin{cases}
D\left(\vecz^{(i)}, \vecz^{(j)}\right) \leq T \quad \Rightarrow \quad \vecy^{(i)} = \vecy^{(j)} \\
D\left(\vecz^{(i)}, \vecz^{(j)}\right) > T  \quad \Rightarrow  \quad \vecy^{(i)} \neq \vecy^{(j)}, \\
\end{cases}
\label{eq:verification}
\end{equation}
where $T$ is a threshold that is optimized by the training dataset. No classifier is needed in this procedure. 
To facilitate subsequent learning by the softmax-based loss function, we use the cosine distance in this work\footnote{ \bflag{We adopt the cosine distance base on the following two considerations.
1) The finite range of the cosine distance, which meets the requirements of the Lipschitz continuous condition in DL training; 
2) For more stable and direct optimization of the distance among RFFs (see Section II).}}, as follows:
\begin{equation}
D\left(\vecz^{(i)}, \vecz^{(j)}\right) = 1- \frac{{\vecz^{(i)}}^T \vecz^{(j)} }{\parallel \vecz^{(i)} \parallel \parallel \vecz^{(j)} \parallel}.
\label{eq:dist}
\end{equation}

The only question that remains is how to learn an RFF extractor $F(\cdot)$ without accessing unknown devices. In the next section, we achieve this goal by presenting a novel model-and-data driven DL framework.

%% file: sections/03_RFF.tex
\begin{figure*}[t]
	\centering
	\includegraphics[width=0.95\linewidth]{\rootpath/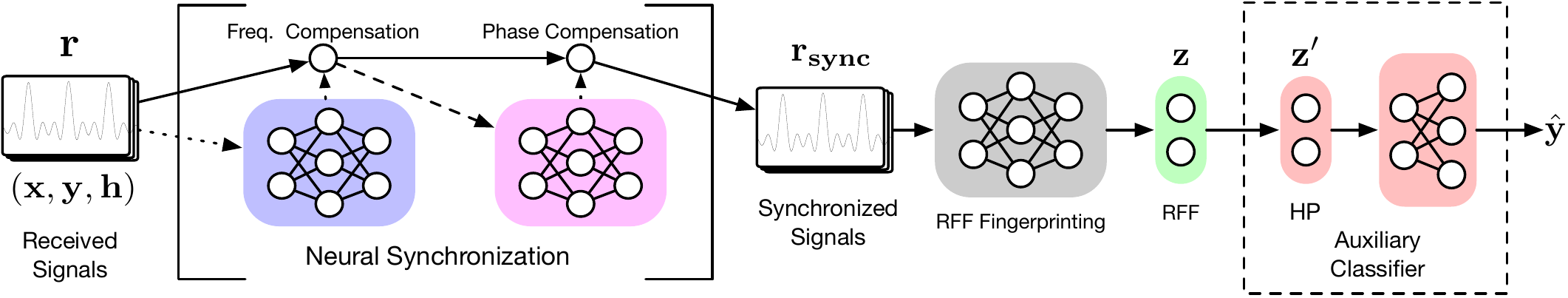}
	\caption{Neural Synchronization for RFF Extraction. The proposed Neural Synchronization~(NS) module is primarily composed of two deep neural networks with the same structure. They are referred to as the frequency offset estimator~(blue) and the phase offset estimator~(pink). The estimated frequency and phase offsets are used to synchronize the baseband signal. The synchronized signal is then used for fingerprint extraction using the neural network RFF extractor~(gray), and finally the auxiliary linear classifier~(red) provides training supervision for end-to-end training.}
	\label{fig:diagram}
\end{figure*}

\section{Neural Synchronization for RFF Extraction}
In this section, we elaborate the framework of the proposed neural synchronization~(NS)-based RFF. This framework combines the advantages of model-based signal processing priors and the data-driven learning ability of DNN. A block diagram of the proposed framework is shown in Fig.~\ref{fig:diagram}.

\subsection{Module Design}
As shown in Fig.~\ref{fig:diagram}, the NS-based RFF extractor consists of three basic neural networks with similar architectures. The first two are the proposed NS module that estimate the phase and frequency offsets of the input signals and output the signals compensated by these offsets. The last neural network is the RFF extractor. 

We begin with the design of the proposed NS module, which generalizes traditional carrier synchronization (TS) techniques used in previous RFF extraction methods~\cite{peng2018design, merchant2018deep, yu2019robust}. TS was originally designed to compensate for the frequency and phase offsets in the received signals. These offsets are often caused by low-cost oscillators at the receiver, which jeopardize the communication quality. In particular, TS performs compensation by first estimating these offsets, typically via maximum likelihood estimation (MLE):
\begin{align}
	 \omega_{\text{TS}}, \phi_{\text{TS}} = \operatorname{arg}\underset{\omega, \phi}{\max}\quad p(\vecr|\omega, \phi),
	 \label{formula:MLE-for-TS}
\end{align}
where $\omega$ and $\phi$ are the frequency and phase offsets respectively, and $\vecr$ is the received signal. In general, the problem in \eqref{formula:MLE-for-TS} can be further rewritten as~\cite{merchant2018deep}:
\begin{equation}
\begin{aligned}
	 \omega_{\text{TS}}, \phi_{\text{TS}} = \operatorname{arg}\underset{\omega, \phi}{\min} \sum_{t=1}^{M} \parallel \vecr(t) - \vecx(t)\exp\{j 2 \pi (\omega t- \phi)\} \parallel^2,
	 \label{formula:impl-for-TS}
\end{aligned}
\end{equation}
where $\vecx$ is the preamble signal, $\vecr$ is the received signal, $\vecx(t)$ and $\vecr(t)$ denote the $t$-th element of $\vecx$ and $\vecr$, respectively, and $M$ is the length of $\vecr$. Once $\omega_{\text{TS}}$ and $\phi_{\text{TS}}$ are obtained, the compensation of the received signal is performed as
\begin{equation}
\begin{aligned}
 \vecr_{\text{TS}}(t) = \vecr(t)\exp\{-j 2 \pi (\omega_{\text{TS}} t- \phi_{\text{TS}})\}.
\end{aligned}
\end{equation}
TS has been shown to be a useful technique for communication tasks. However, the offsets themselves may provide information about the device itself. Simply removing these offsets in the received signal may lead to a loss of information about the device identity when extracting RFFs from the compensated signals.  

To this end, we develop a neural network generalization of the synchronization process that can automatically determine how to perform compensation in a data-driven manner. More specifically, rather than performing offset estimation as in \eqref{formula:impl-for-TS}, which removes information about the identity of the devices, we propose to estimate the offsets by two deep neural networks trained with data. Let $F_{\theta_{\omega}}(\cdot)$ and $F_{\theta_{\phi}}(\cdot)$ be the deep neural networks used in the estimation of the frequency and phase offsets, respectively. The compensation needed for RFF extraction is performed as follows:
\paragraph{Neural Frequency Compensation}
We first adopt a neural network $F_{\theta_{\omega}}$ to estimate the frequency offset $\omega_{\text{NS}}$:
\begin{equation}
\begin{aligned}
\omega_{\text{NS}} = F_{\theta_{\omega}}(\vecr),
\label{eq:freq}
\end{aligned}
\end{equation}
where $\theta_{\omega}$ contains the parameters (i.e., weights and biases) of the neural network. \bflag{We call $\omega_{\text{NS}}$ the \emph{device-irrelevant frequency offset} since it is aimed at device identification.} Given $\theta_{\omega}$, frequency compensation of the received signal is performed by 
\begin{equation}
\vecr_{\omega}(t) = \vecr(t)\exp\{-j 2 \pi \omega_{\text{NS}} t\}.
\end{equation}
\paragraph{Neural Phase Compensation}
Similarly, \bflag{the \emph{device-irrelevant phase offset}}, denoted by $\phi_{\text{NS}}$, is estimated by another neural network $F_{\theta_{\phi}}$ using $\vecr_{\omega}(t)$:
\begin{equation}
\begin{aligned}
\phi_{\text{NS}} = F_{\theta_{\phi}}(\vecr_{\omega}),
\label{eq:phase}
\end{aligned}
\end{equation}
where $\theta_{\phi}$ includes the parameters of the phase neural network. The synchronized signal $\vecr_{\text{sync}}$ is then obtained by computing
\begin{equation}
\begin{aligned}
 \vecr_{\text{NS}}(t) = \vecr_{\omega}(t)\exp\{j 2 \pi \phi_{\text{NS}}\}.
\label{eq:syncr}
\end{aligned}
\end{equation}

The two networks, i.e., $F_{\theta_{\omega}}(\cdot)$ and $F_{\theta_{\phi}}(\cdot)$, constitute the NS module, which compensate for device-irrelevant phase and frequency offsets in the received signal. 
After compensation, the signal $\vecr_{\text{NS}}$ is sent to a third neural network  $F_{\theta_{\text{RFF}}}(\cdot)$ with parameters  $\theta_{\text{RFF}}$ to compute the desired RFF:
\begin{equation}
\begin{aligned}
\vecz = F_{\theta_{\text{RFF}}}(\vecr_{\text{NS}}).
\end{aligned}
\end{equation}
We denote the entire ``neural frequency compensation + neural phase compensation + RFF computation'' procedure above as:
\begin{equation}
\begin{aligned}
\vecz = F_{\Theta}(\vecr),
\end{aligned}
\end{equation}
where $F_{\Theta}(\cdot)$ is the proposed NS-based RFF extractor, and $\Theta = \{\theta_{\omega}, \theta_{\phi}, \theta_{\text{RFF}}\}$ denotes the parameters of the entire RFF extraction network. The parameters in $\Theta$ will be learned jointly using an objective function that will be described later. 

\begin{figure*}[t]
	\centering
	\includegraphics[width=\linewidth]{\rootpath/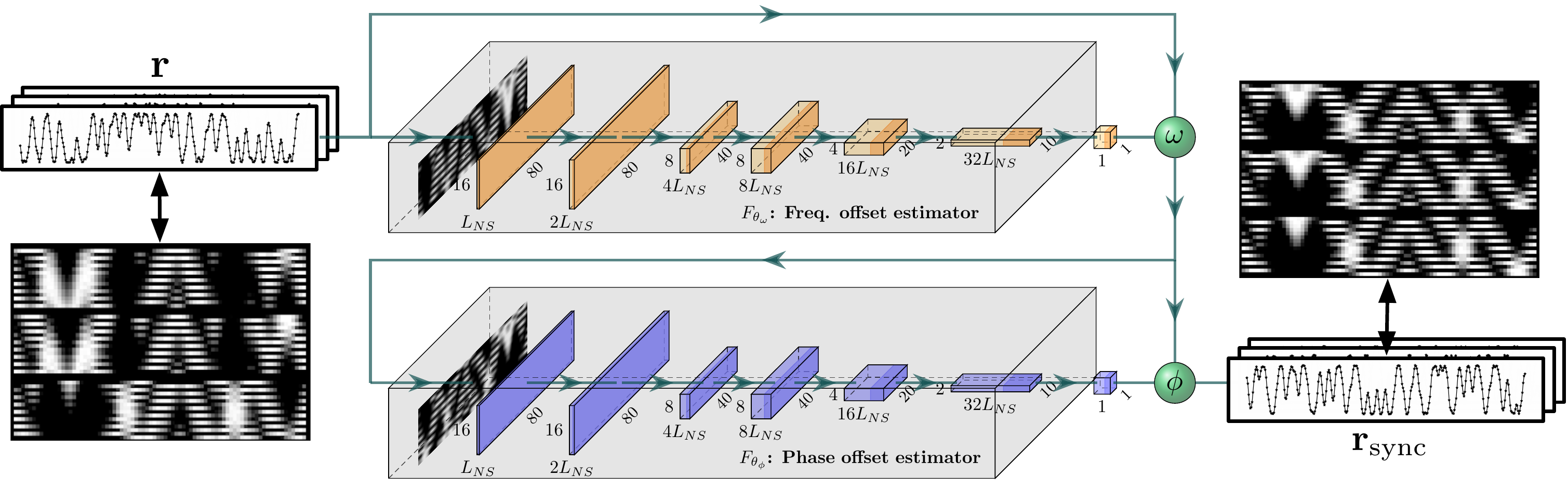}
	\caption{Diagram of the proposed Neural Synchronization module.}
	\label{fig:nn}
\end{figure*}

We propose to use a \emph{convolutional neural network} (CNN) for all the three networks $F_{\theta_{\omega}}(\cdot), F_{\theta_{\phi}}(\cdot)$, and $F_{\theta_{\text{RFF}}}(\cdot)$. The motivation for using the convolution operation is that the preamble signal often exhibits high periodicity (e.g., the ideal preamble signal is often comprised of several identical symbols, as in the IEEE 802.15.4 standard~\cite{ieee2015ieee}.), and convolution is excellent at extracting periodic patterns~\cite{vanwavenet}. In principle, 1D convolution can be used for this purpose; however, it cannot capture cross-period patterns that $\vecr$ exhibits, which may be useful for identifying different devices. We therefore propose the use of 2D convolution to extract a richer set of patterns. As shown in Fig.~\ref{fig:nn}, there are two components in the proposed 2D convolution network (which we refer to as basic CNN~(BCNN) from now on): a signal-to-image layer and a series of convolution layers. The complexity of these layers can be controlled by some hyper-parameters as summarized in Table~\ref{tb:nn}. 

\emph{i) Signal-to-Image Layer}: This layer converts the original 1D signal into a 2D image to facilitate the subsequent processing. Formally, given an input signal $\vecr \in \mathbb{C}^{M}$, this layer computes its image representation $\mathbf{I} \in \R^{2\times \frac{M}{S}\times S}$ as:
\begin{equation}
\mathbf{I}_{1,i,j}=\Re\left\{\mathbf{R}_{i,j}\right\}, \qquad 
\mathbf{I}_{2,i,j}=\Im\left\{\mathbf{R}_{i,j}\right\},
\end{equation}
\begin{equation}
\mathbf{R}= \left[
 \begin{matrix}
   \vecr_1 & \vecr_2 & \cdots & \vecr_{S-1} &\vecr_S \\
   \vecr_{S+1} & \vecr_{S+2} & \cdots  & \vecr_{2S-1} &\vecr_{2S} \\
   \cdots & \cdots & \cdots & \cdots & \cdots \\
   \vecr_{M-S+1} & \vecr_{M-S+2} & \cdots & \vecr_{M-1} &\vecr_M 
  \end{matrix}
  \right],
  \label{formula:signal2image}
\end{equation}
where $\mathbf{R} \in \C^{\frac{M}{S}\times S}$ and $\Re\{\cdot\}$ and $\Im\{\cdot\}$ are, respectively, the real and imaginary parts of the input, and $\vecr_i$ is the $i$-th element of the received signal $\vecr$. The two channels of the image correspond to the real and imaginary parts of the signal, with each pixel $\mathbf{I}_{c,i,j}, c\in\{1,2\}$, representing one dimension of the input signal $\vecr$. 
The width of the image, denoted by $S$, is set such that each row in the image corresponds to one-half a symbol period~(i.e., 16 chips in the IEEE 802.15.4 standard). Therefore pixels in the same row are from the same symbol, whereas pixels belonging to disjoint rows come from different symbols. 

\emph{ii) Convolution Layers}: With the image representation in \eqref{formula:signal2image}, we are ready able to perform 2D convolution to extract both intra-period and inter-period patterns. Mathematically, given an image $\mathbf{I} \in \mathbb{R}^{C_{\mathbf{I}} \times H_{\mathbf{I}} \times W_{\mathbf{I}}}$ and a kernel $\mathbf{K} \in \mathbb{R}^{C_{\mathbf{K}} \times H_{\mathbf{K}} \times W_{\mathbf{K}} }$, the 2D convolution $\mathbf{I} * \mathbf{K}$ is defined as:
\begin{equation}
(\mathbf{I} * \mathbf{K})_{i, j}=\sum_{i_{h}=1}^{H} \sum_{i_{w}=1}^{W} \sum_{i_{c}=1}^{C} \mathbf{K}_{i_{c}, i_{h}, i_{w}} \mathbf{I}_{i_{c},i+i_{h}-1, j+i_{w}-1}.
\end{equation}
\bflag{Inspired by~\cite{simonyan2014very}, we adopted small convolutional filters with size $3\times3$ in our convolutional neural networks. The $3\times3$ filters are the smallest available for capturing the 2D correlations of an image. By stacking the small filter layers, we can use fewer parameters to achieve the same effective receptive field as filters with larger size~\cite{simonyan2014very}.} As in many other deep learning frameworks~\cite{yu2017unsupervised, he2015delving,he2016identity}, we consecutively apply batch normalization~(BN) and leaky ReLU activation~(LReLU) after the convolution operation. The former helps to stabilize and accelerate training, whereas the latter serves as a nonlinear transformation in the network. The three operations together form a layer, denoted by $g_{n_g}$, in the convolution neural network, where $n_g$ is the layer index. We repeatedly apply these operations to extract high-level patterns from the input image $\mathbf{I}$, resulting in a series of convolution layers $g = g_1 \circ g_2 \circ ... \circ g_{N_{g}-1}$ with $N_{g}-1$ layers. A fully connected layer is appended to these convolution layers to compute the final output $\omega$, $\phi$, or $\vecz$. 

\bflag{Note that the number of layers depends on the size of the input image. For example, the neural network structure we use in our experiments~(with input size $16\times 80$) is presented in Table \ref{tb:nn}. After applying six convolutional layers, the output size for the 6th layer is too small to apply additional convolution (i.e., smaller than $3\times3$); therefore we stop convolving and instead adopt a fully connected layer, which is also the final layer of the CNN. The number of filters are controlled by a hyperparameter $L$, which is used for model complexity adjustment.}

To more intuitively explain the benefits of NS, we visualize the real part of the synchronized signals from three different devices in Fig.\ref{fig:vis}. In contrast to TS, which directly eliminates the distinctions between these signals, the proposed NS module can preserve device information while maintaining signal alignment.  

 \begin{figure*}[hbt]
	\centering
    \includegraphics[width=0.95\linewidth]{\rootpath/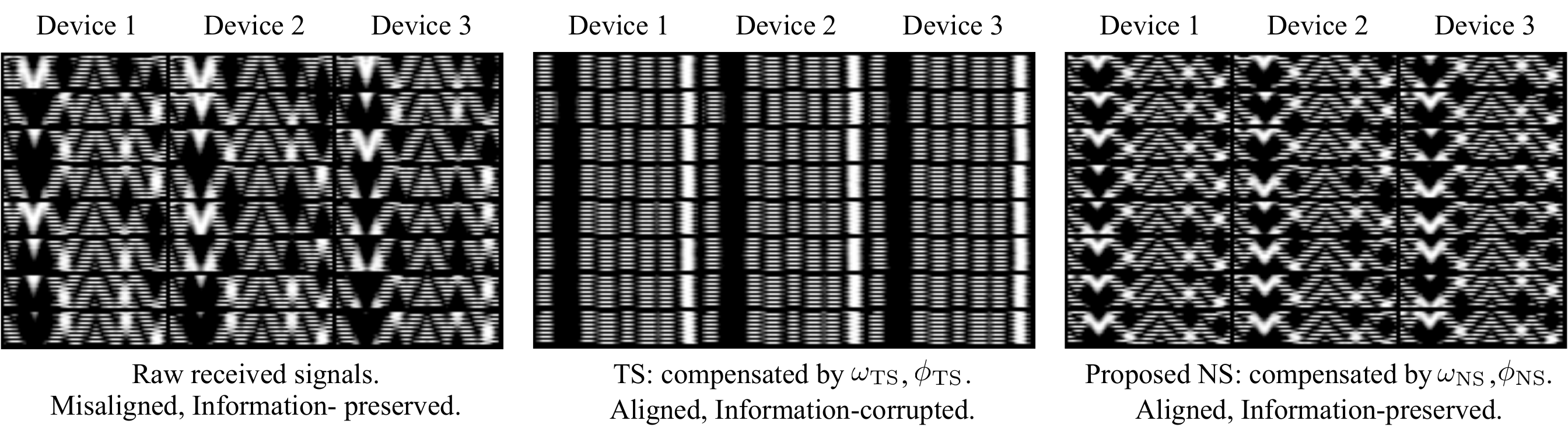}
	\caption{Visualization of synchronized signals. }
	\label{fig:vis}
\end{figure*}
\input{tables/nn}

\subsection{Learning Algorithm}
Given the proposed NS-based RFF extractor $ F_{\Theta}(\cdot)$, we switch our focus to training the above network $ F_{\Theta}(\cdot)$ with the collected data $\{(\vecr^{(n)}, \vecy^{(n)})^N_{n=1}\}$. Ideally, we wish that our learning algorithm and objective can maximally separate RFFs for different devices. We demonstrate how to achieve this goal below with a novel objective, which is motivated by the limitations of a particular na\"ive.

\paragraph{A Na{\"i}ve Solution} A vanilla approach for learning $ F_{\Theta}(\cdot)$ to optimize the  objective~\eqref{eq:origin_proplem} can be found by defining the distribution $p_{\mathbf{W}}(\vecy|\vecz)$ as a softmax distribution:

\begin{align}
\label{eq:score-original}
p_{\mathbf{W}}(\vecy_i | \vecz) &= \frac{\exp\{ \vecw^{\top}_i \vecz  \}}{\sum_j \exp\{ \vecw^{\top}_j \vecz  \}},
\end{align}
where $\vecz$ is the RFF and $\W = \{\{\vecw_j \}_{j=1}^K\}$ are the parameters of the softmax score. With this construction, we can now optimize the objective in (2) to learn $\W$. 

There is one major issue with the above na\"ive approach: maximizing the classification score does not necessarily produce a good distance between features. This is because the classification score in \eqref{eq:score-original} can be pushed to be arbitrarily high by enlarging the norm of some of the obtained RFF, as demonstrated by the following proposition:
\begin{proposition} 1
For all $\vecz$ that satisfy $\vecw^{\top}_i \vecz > \vecw^{\top}_j \vecz$, $\forall j\neq i$, and any $\lambda> 1$, we have
\begin{align}
\frac{\exp\{ \lambda \vecw^{\top}_i \vecz  \}}{\sum_j \exp\{ \lambda \vecw^{\top}_j \vecz \}} \ge  \frac{\exp\{   \vecw^{\top}_i \vecz  \}}{\sum_j \exp\{ \vecw^{\top}_j \vecz \}}.
\end{align}
\end{proposition}
\begin{proof}
Dividing both sides of the inequality by $\exp\{\lambda \vecw^{\top}_i \vecz\}$, we can obtain
\[
\frac{1}{\sum_j \exp\{ \lambda (\vecw^{\top}_j \vecz - \vecw^{\top}_i \vecz)\}} \ge  \frac{1}{\sum_j \exp\{ \vecw^{\top}_j \vecz - \vecw^{\top}_i \vecz \}},
\]
where $\lambda > 1$, $\sum_j \exp{x}$ is monotonically increasing, and $w_j^T z - w_i^T z < 0$ for all $j \ne i$. Therefore, the inequality is directly established.
\end{proof}
The above proposition implies that the classifier can easily achieve a high classification score by manipulating the norms of the learned RFF. For example, we can increase the norms of RFF whose classification score is high and decrease the norms of those RFF whose classification score is low. Therefore, although this method is optimal in terms of the classification score, the learned RFFs are not necessarily well-separated in the space (note that in open-set settings we can only classify RFFs by their distance). As a result, the features obtained through the above na\"ive training may not be suitable for open-set RFF authentication.

\paragraph{Hypersphere Projection} To overcome the above issue in the softmax classifier, we propose to use a \emph{hyperspherical representation} for the learned RFF, inspired by recent advance in face recognition~\cite{wen2016discriminative,ranjan2017l2,deng2019arcface}. The RFF in our framework does not lie in Euclidean space but is on the surface of a hypersphere. These hyperspherical RFFs are obtained by applying a hyperspherical projection on the original RFF, namely:
\begin{equation}
\begin{aligned}
	\vecz^\prime = \alpha \frac{\vecz}{\parallel \vecz \parallel},
\label{eq:hp}
\end{aligned}
\end{equation}
where $\alpha >0$ is the radius of the hypersphere, which is left as a hyperparameter. With this hyperspherical representation, the classification probability $q_{\W}(\vecy_i|\vecz)$ is now (re-)defined as 
\begin{align}
\label{eq:score-hp}
q_{\mathbf{W}}(\vecy_i | \vecz) &= \frac{\exp\{ \vecw'^{\top}_i \vecz'\}}{\sum_j \exp\{ \vecw'^{\top}_j \vecz' \}},
\end{align}
where
\begin{equation}
    \vecw' = \frac{\vecw }{\|\vecw\|}.
\end{equation}
\bflag{Now, as both $\vecw'$ and $\vecz'$ in \eqref{eq:score-hp} have fixed norms, maximizing \eqref{eq:score-hp} is equivalent to minimizing the cosine distance between $\vecw'$ and $\vecz'$, hence minimizing the cosine distance between all $\vecz'$ belonging to the same device. In other words, the hyperspherical representation can guarantee that the optimization of $q_{\mathbf{W}}(\vecy|\vecz)$ and the cosine distances coincide with each other.} In the following experiments, we will see that this hyperspherical representation is also a key step for improving the performance of the proposed RFF framework. 

Finally, given the proposed NS-based RFF extractor $F_{\Theta}(\cdot)$, the auxiliary classifier in (\ref{eq:score-hp}), and the training set $\{(\vecr^{(n)}, \vecy^{(n)})_{n=1}^N\}$, the reformulated learning objective of (\ref{eq:origin_proplem})-(\ref{eq:origin_condition})  becomes
\begin{equation}
\begin{aligned}
	\label{eq:obj} \underset{\Theta, \W}{\operatorname{min}} \quad &\mathcal{L}(\Theta, \W)\\
	 &= - \frac{1}{N} \sum_{n=1}^{N}\ln q_{\W}\left(\vecy^{(n)} | F_{\Theta}(\vecr^{(n)})\right).
\end{aligned}
\end{equation}

The end-to-end training algorithm is summarized in Algorithm~\ref{alg:NS}. \bflag{Note that the auxiliary classifier $\mathbf{W}$ here is to provide training supervision for more discriminative RFFs. Once the training is done, this classifier will be discarded. Only the RFF extractor is required for RFF authentication~(since we compare RFFs by their pairwise distance).}

\input{tables/train_algorithm.tex}

%% file: tables/nn.tex
\begin{table}[!t]  

\caption{The structure of basic-CNN}
\centering
\begin{tabular}{crc}  

\toprule
\multicolumn{3}{l}{{\bf HyperParams}: Image width $S$, complexity $L$} \\  
\midrule  
\multicolumn{3}{l}{{\bf Input}: Signal $\vecr \in \mathbb{C}^{M}$ $\rightarrow$ Image $\mathbf{I} \in \mathbb{R}^{2\times \frac{M}{S} \times S}$} \\  
\midrule
\multicolumn{3}{l}{\bf Convolution layers}  \\ 
  {\bf Layers} & {\bf Filters/Stride/Padding} & {\bf Activation}\\
  1 & $L\times 3\times3/1/1$   & BN + $\text{LReLU}_{(0.2)}$   \\  
  2 & $2L\times 3\times3/1/1$  & BN + $\text{LReLU}_{(0.2)}$   \\  
  3 & $4L\times 3\times3/2/1$  & BN + $\text{LReLU}_{(0.2)}$   \\  
  4 & $8L\times 3\times3/1/1$  & BN + $\text{LReLU}_{(0.2)}$   \\  
  5 & $16L\times 3\times3/2/1$ & BN + $\text{LReLU}_{(0.2)}$   \\  
  6 & $32L\times 3\times3/2/1$ & BN + $\text{LReLU}_{(0.2)}$   \\ 
\midrule
\multicolumn{3}{l}{{\bf Output}: FC($\frac{LM}{2}$, output dimension)}\\
\bottomrule  

\end{tabular}
\label{tb:nn}
\end{table}

%% file: tables/train_algorithm.tex
\begin{algorithm}[t]
  \caption{NS-based RFF Training Algorithm}
  \label{alg:NS}
\begin{algorithmic}
  \STATE {\bfseries Input:} Received preamble signal with device label pairs $\{(\vecr^{(i)}, \vecy^{(i)})\}_{i=1}^{N}$ constructed by $\mathscr{F}$; the complexity parameters $L_{\text{NS}}$ and $L_{\text{RFF}}$.
  \STATE {\bfseries Output:} $\Theta^* = \{\theta_{\omega}^*, \theta_{\phi}^*, \theta_{\text{RFF}}^*\}$
  \STATE {\bfseries Hyperparam:} learning rate $\eta$, hypersphere radius $\alpha$
  \REPEAT
  \FOR{$i=1$ {\bfseries to} $N$}
  \STATE compute  $\vecz^{(i)} = F_{\Theta}(\vecr^{(i)})$
  \STATE compute $q_{\W}(\vecy^{(i)} | \vecz^{(i)})$;
  \ENDFOR
  \STATE compute $\mathcal{L}=-\frac{1}{N} \sum_{i=1}^{N} \log q_{\W}\left(\vecy^{(i)} | \vecz^{(i)}\right)$;
  \STATE $\theta_{\omega} \leftarrow \theta_{\omega}-\eta \nabla_{\theta_{\omega}} \mathcal{L}$;
  \STATE $\theta_{\phi} \leftarrow \theta_{\phi}-\eta \nabla_{\theta_{\phi}} \mathcal{L}$;
  \STATE $\theta_{\text{RFF}} \leftarrow \theta_{\text{RFF}}-\eta \nabla_{\theta_{\text{RFF}}} \mathcal{L}$;
  \STATE $\W \leftarrow \W-\eta \nabla_{\W} \mathcal{L}$;
  \UNTIL{convergence}
  \STATE \textbf{return} $F_{\Theta}$.
\end{algorithmic}
\end{algorithm}

%% file: sections/04_analysis.tex
\section{Theoretical Analysis}
In this section, we illustrate the effectiveness of the proposed NS-based framework from an information-theoretic perspective. Let us begin with a Markov chain that describes a regular RF fingerprinting process:
\begin{equation}
  (\vecy, \vecx, \vech) \leftrightarrow \vecr \rightarrow \vecr^\prime \rightarrow \vecz,
\label{eq:markov}
\end{equation}
where $\vecy$ is the identity of the device, $\vecx$ is the preamble, $\vech$ is the wireless channel, $\vecr$ is the received signal, and $\vecr^\prime$ is the processed version of $\vecr$. Symbol $\vecz$ is the RFF extracted from $\vecr^\prime$. Therefore, given $\vecr^\prime$, the received signal $\vecr$ and RFF $\vecz$ are independent. Then according to the data processing inequality~\cite{mackay2003information}, we have:
\begin{equation}
  I(\vecr; \vecy) \ge I(\vecr^\prime; \vecy) \ge I(\vecz; \vecy),
\label{eq:dpi}
\end{equation}
where $I(\cdot; \cdot)$ is mutual information, and the equality is only achieved if $\vecz$ and $\vecr'$ are both sufficient statistics of $\vecr$ with respect to $\vecy$. The inequalities in \eqref{eq:dpi} imply that the introduction of any preprocessing step $\vecr \to \vecr'$  will inevitably lead to a loss of information about $\vecy$. 
Traditional synchronization~(TS) is designed for recovering $\vecx$, which is not directly relevant to extracting the device hardware information of $\vecy$ from $\vecr$. Therefore, such $\vecr^\prime$ from TS is not necessarily a sufficient statistic for the RFF identification. Let $\vecr^\prime = F_{\text{pre}}(\vecr)$ be any preprocessing that is not aimed at identity extraction, and define the information cost as $C\triangleq  I(\vecr; \vecy) - I(\vecz; \vecy)$. Then we rewrite \eqref{eq:dpi} as 
\begin{align}
	C \ge I(\vecr; \vecy) - I(F_{\text{pre}}(\vecr);\vecy) >0.
\end{align}
This inequality indicates that no matter how powerful the learning model that is applied in subsequent fingerprint extraction, a certain amount of information loss exists due to the inappropriate preprocessing.

By contrast, the proposed NS-based RFF extractor $F_{\Theta}$ is a unified trainable model that combines the procedures of preprocessing and RFF extraction. The corresponding Markov chain in \eqref{eq:markov} can be simplified as 
\begin{equation}
  (\vecy, \vecx, \vech) \leftrightarrow \vecr \overset{F_{\Theta}}{\longrightarrow} \vecz, 
\label{eq:markov2}
\end{equation}
and the corresponding inequality as follows:
\begin{equation}
  C=I(\vecr; \vecy) - I(\vecz; \vecy)\ge0.
\end{equation}
In this setting, the information cost $C$ can be arbitrarily close to 0 by directly maximizing $I(\vecz; \vecy)$. In fact, the proposed end-to-end training with the learning objective in (\ref{eq:obj}) is equivalent to maximizing a lower bound on $I(\vecz; \vecy)$, as demonstrated by the following theorem.
\begin{theorem} 1
Given the Markov chain in (\ref{eq:markov2}), the RFF extractor $p_{\Theta}(\vecz|\vecr)=\delta(\vecz-F_{\Theta}(\vecr))$, auxiliary classifier $q_{\W}(\vecy|\vecz)$, and the training dataset distribution $p(\vecy, \vecr|\mathscr{F})$, the variational lower bound of $I(\vecz; \vecy)$ is given by
\begin{equation}
\begin{aligned}
	I(\vecz; \vecy) \ge -\mathcal{L}_v \triangleq  \underset{\vecy, \vecr\sim p(\vecy, \vecr|\mathscr{F})}{\E} [\int d \vecz & p_{\Theta}(\vecz | \vecr) \ln q_{\vecw}(\vecy | \vecz)],
	\label{eq:lower}
\end{aligned}
\end{equation}
with the equality if and only if $\mathcal{D}_{\text{KL}}[p(\vecy | \vecz)||q_{\W}(\vecy | \vecz)]=0$, where $\mathcal{D}_{\text{KL}}$ is the Kullback-Leibler divergence. 
\label{thm:1}
\end{theorem}
\begin{proof}
See Appendix A. 
\end{proof}

Note that $p_{\Theta}(\vecz|\vecr)=\delta(\vecz-F_{\Theta}(\vecr))$ is deterministic, while $p(\vecy, \vecr| \mathscr{F})$ can be approximated by the empirical data distribution $p(\vecy, \vecr|\mathscr{F}) = \frac{1}{N} \sum_{n=1}^{N} \delta_{\vecy^{(n)}}(\vecy) \delta_{\vecr^{(n)}}(\vecr)$. Then $\mathcal{L}_v$ can be approximated by
\begin{equation}
\begin{aligned}
  \mathcal{L}_{v}(\Theta, \vecw, \mathscr{F})
                  &\approx - \frac{1}{N} \sum_{n=1}^{N} \E_{p_{\Theta}\left(\vecz | \vecr^{(n)}\right)} \bigg[ \ln q_{\W}\left(\vecy^{(n)} | \vecz\right) \bigg]\\
                  &= - \frac{1}{N} \sum_{n=1}^{N} \ln q_{\W}\left(\vecy^{(n)} | F_{\Theta}(\vecr^{(n)}) \right),
\end{aligned}
\end{equation}
which is the proposed learning objective in (\ref{eq:obj}). This means that the proposed framework can better optimize the mutual information between device identities and RFFs than handcrafted preprocessing methods. This mutual information directly reflects the quality of the feature learning, which further influences its generalizability to unknown devices. This is why the proposed learning algorithm is necessary for open-set RFF authentication.

%% file: sections/05_experiments.tex
\input{tables/baseline}

\begin{figure}[t]
	\centering
	\includegraphics[width=0.99\linewidth]{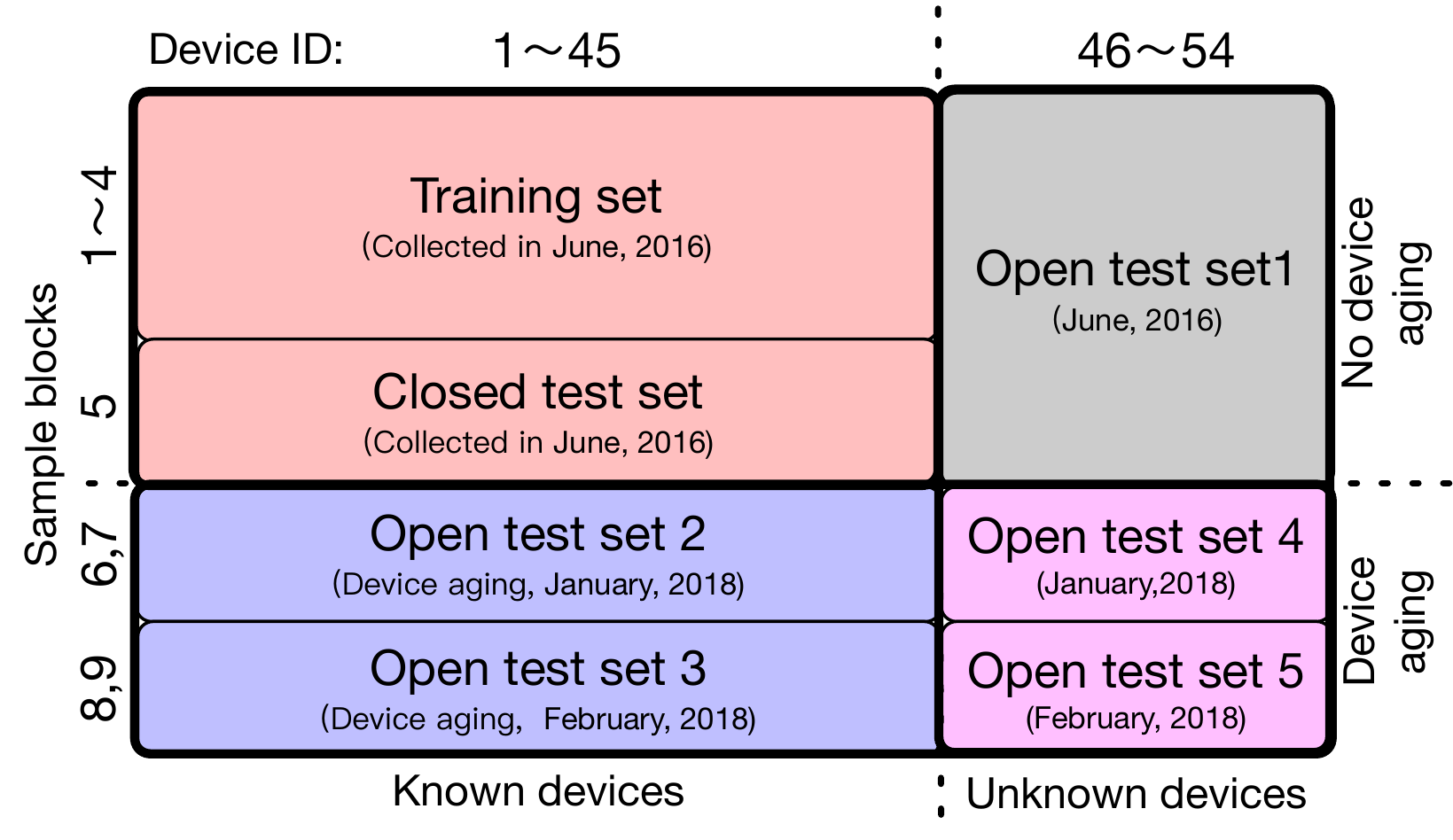}
	\caption{Data segmentation diagram.}
	\label{fig:dataset}
\end{figure}

\section{Experimental Evaluation}
In this section, we conduct a series of experimental tests to verify the effectiveness of the proposed NS-based RFF. We compare the performance of the proposed framework with that of conventional model-driven~(TS-based RFF) and data-driven methods~(pure DL-based RFF). For a more comprehensive evaluation, we divide the experiments into \bflag{four} parts: 1)~Performance comparison for closed and open test sets; 2)~Evaluating the performance for different signal-to-noise ratios~(SNRs); 3)~Performance comparison for different network complexity between the proposed NS-based RFF extractor and the pure DL-based RFF extractor; \bflag{4) Performance comparison of TS-based RFF with the frequency and phase offsets as additional features.}

The source codes for the proposed NS approach are implemented in Pytorch with the DL research toolbox {\bf{MarverToolbox}}. Note that the source codes are open and available at~\cite{Xie2019}, and {\bf{MarverToolbox}} is an open-source toolbox developed on our own for GPU acceleration of complex tensor computations and facilitating DL communication research, and is available at~\cite{Xie2019a}.

\subsection{Experimental Setup}

\begin{figure*}[t]
	\centering
	\subfigure[Closed test set: Known devices/no device aging.]{
        \centering
        \begin{minipage}[t]{0.32\linewidth}
        \centering
        \includegraphics[width=\linewidth]{\rootpath/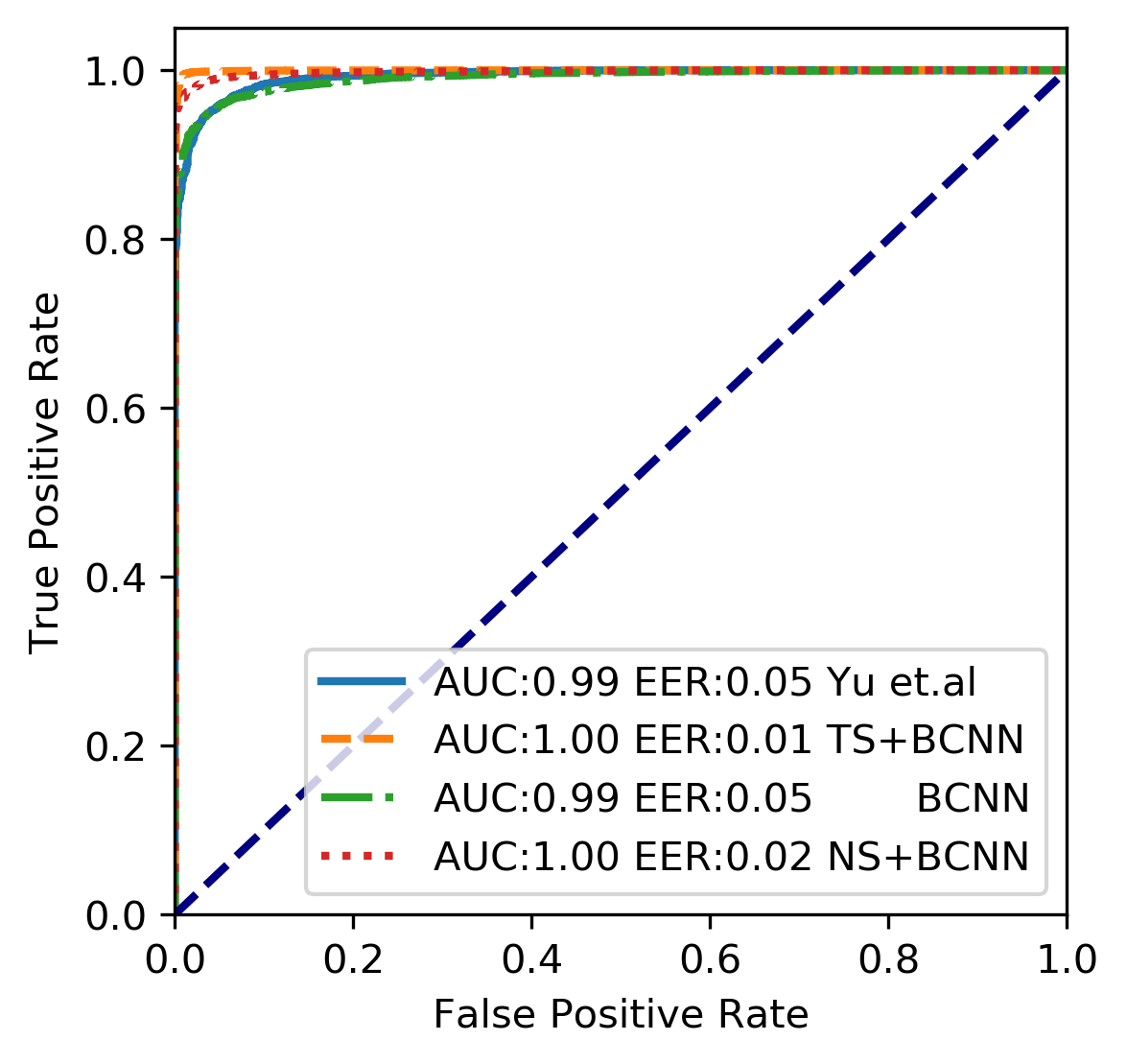}
        \end{minipage}
    }%
    \subfigure[Open 2: Known devices/device aging.]{
        \centering
        \begin{minipage}[t]{0.32\linewidth}
        \centering
        \includegraphics[width=\linewidth]{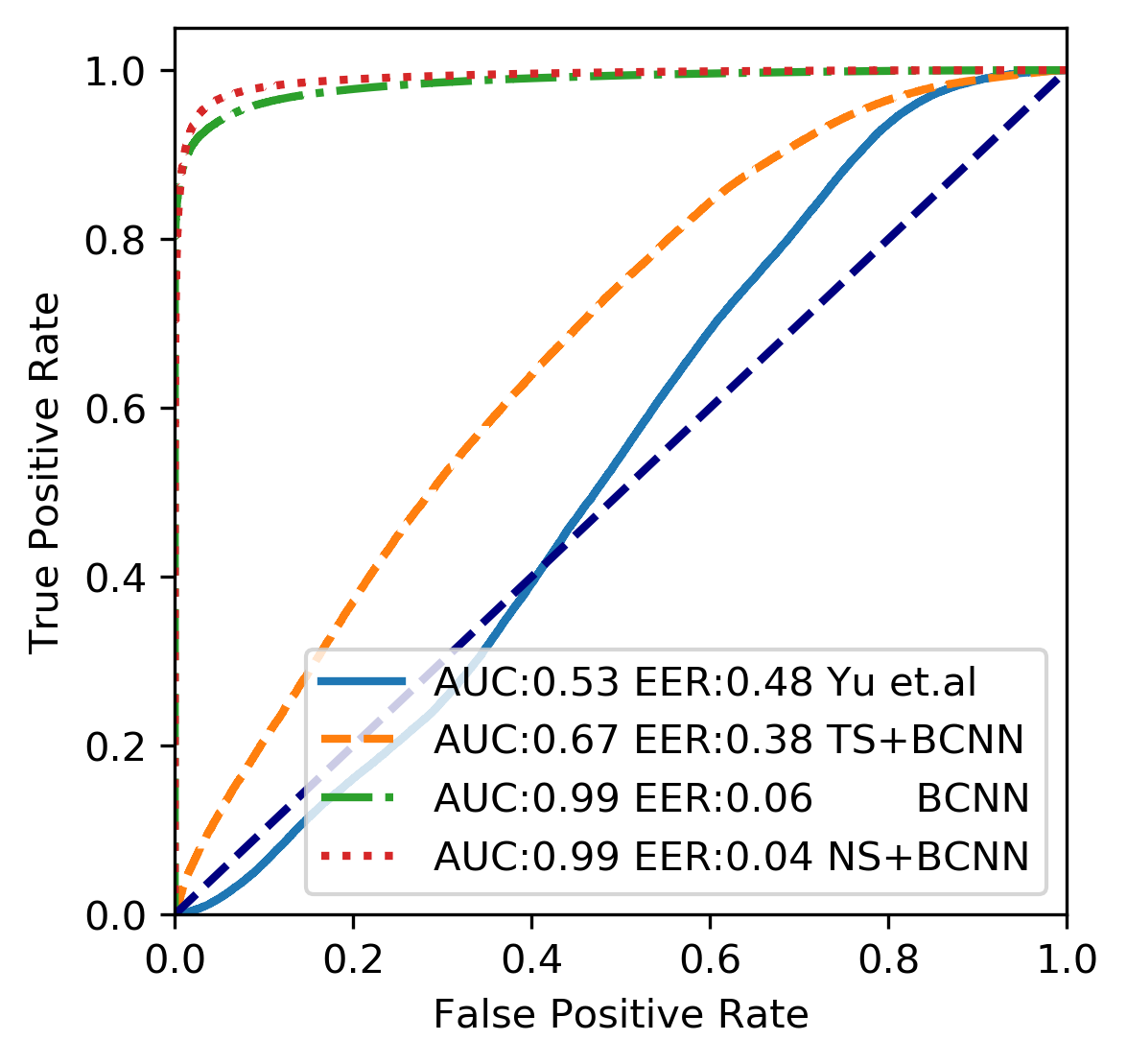}
        \end{minipage}
    }%
    \subfigure[Open 2-3: Known devices/device aging.]{
        \centering
        \begin{minipage}[t]{0.32\linewidth}
        \centering
        \includegraphics[width=\linewidth]{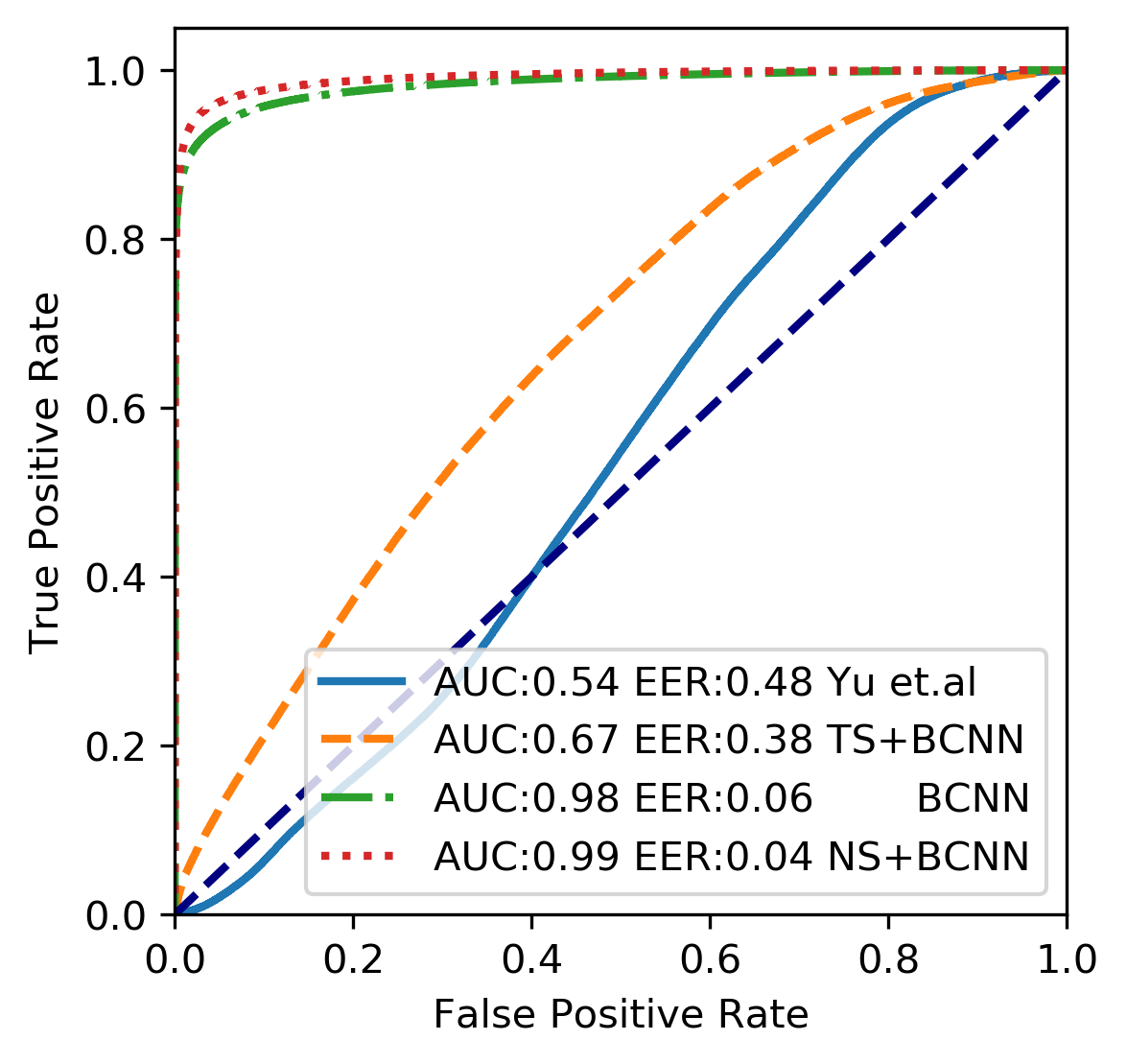}
        \end{minipage}
    }\\
    \subfigure[Open 1: Unknown devices/no device aging.]{
        \centering
        \begin{minipage}[t]{0.32\linewidth}
        \centering
        \includegraphics[width=\linewidth]{\rootpath/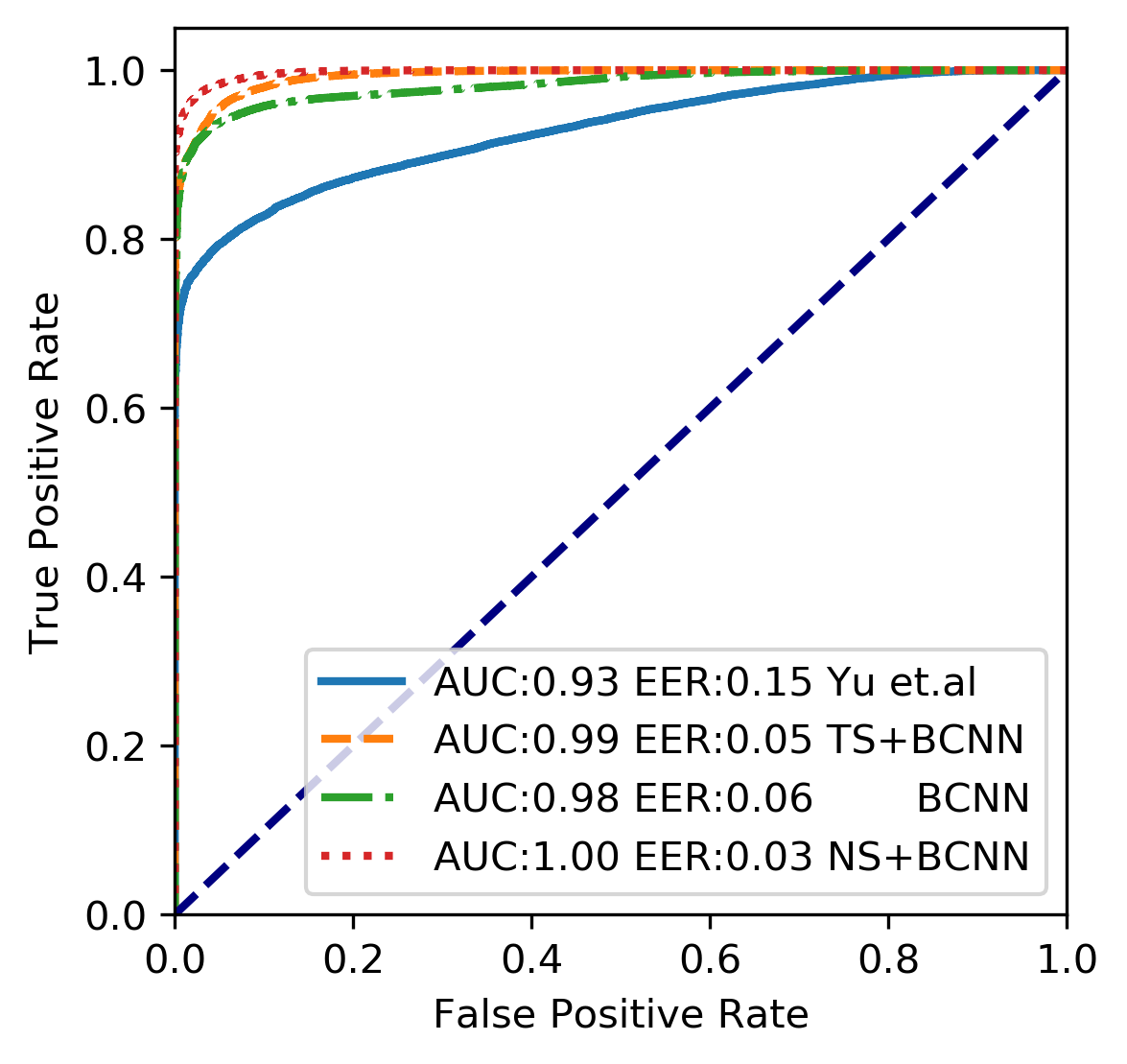}
        \end{minipage}
    }%
    \subfigure[Open 4: Unknown devices/device aging.]{
        \centering
        \begin{minipage}[t]{0.32\linewidth}
        \centering
        \includegraphics[width=\linewidth]{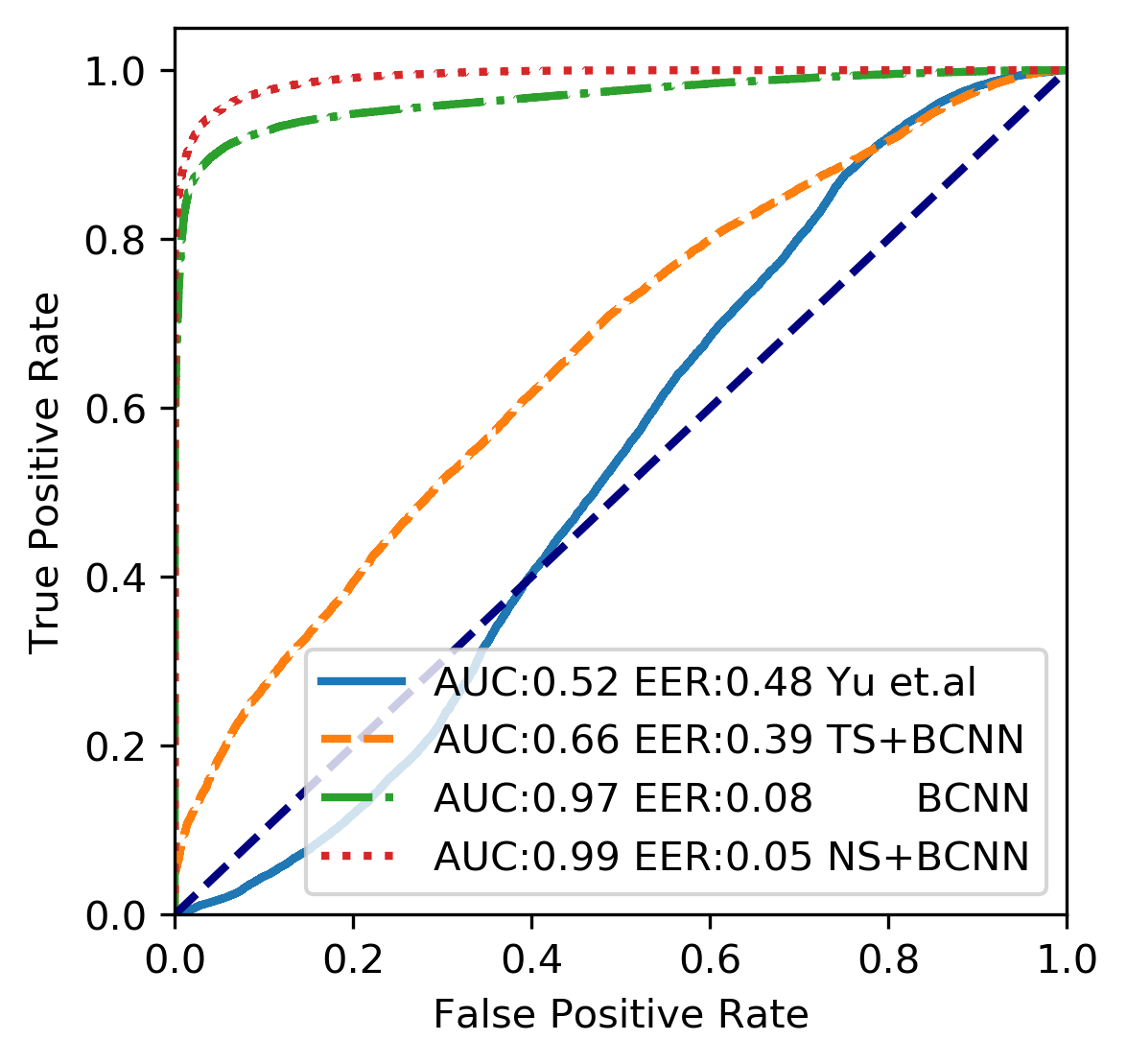}
        \end{minipage}
    }%
    \subfigure[Open 4-5: Unknown devices/device aging.]{
        \centering
        \begin{minipage}[t]{0.32\linewidth}
        \centering
        \includegraphics[width=\linewidth]{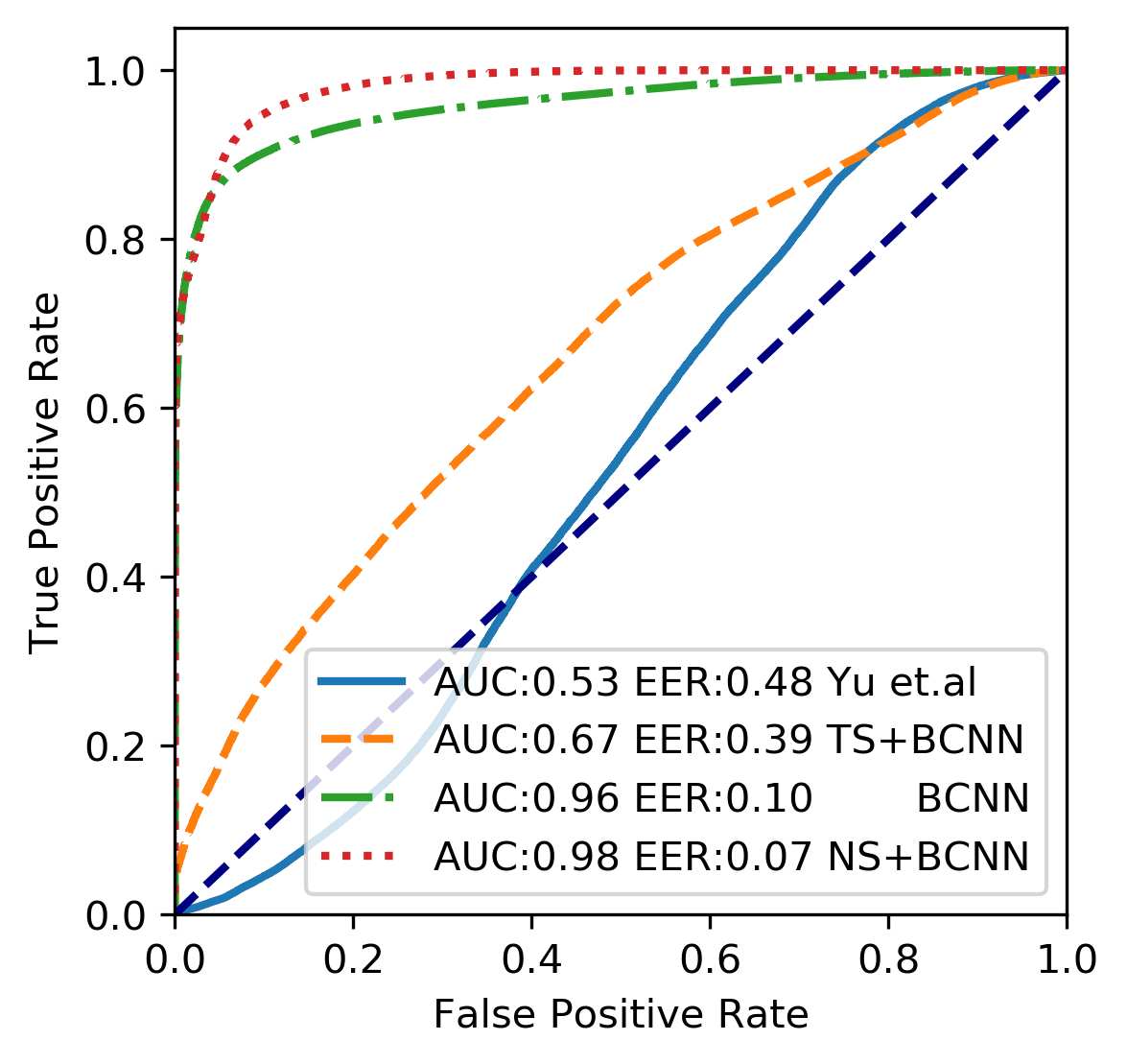}
        \end{minipage}
    }\\
	\caption{ROC curves of different methods under closed and open test sets (SNR = 30 \bflag{dB}).}
	\label{fig:exp0}
	
\end{figure*}

\paragraph{Dataset}We collected data from 54 TI CC2530 ZigBee devices using a USRP N210 as the receiver \bflag{and IEEE 802.15.4 as the physical layer standard.} The ZigBee devices transmitted at a maximum power of 19 dBm and were located within one meter from the receiver. \bflag{The experimental system operates at 2.4 GHz frequency band with a USRP sampling rate of 10 Msample/s. Each preamble signal contains 1280 sample points (i.e., the dimension of the data set), and the energy is normalized to one unit. All the data sets are collected in a real demo testbed, thus the received data are obtained with inevitable and practical noise level~($\text{SNR} \approx 30$ dB).}

As shown in Fig.~\ref{fig:dataset}, the dataset consists of 8 sample blocks. Blocks 1-5 were sampled within the same day~(without device aging\footnote{\bflag{The term ``device aging'' here refers to the degradation of the device performance over time. The devices we used for collecting data had been operating for 18 months without interruption. }}), while blocks 6-9 were sampled after 18 months~(with device aging). Among them, blocks 6-7 were sampled on the same day, and the blocks 8-9 were sampled on another day.
\bflag{The extended data collection interval ensures data independence and helps better verify the generalizability and robustness of the proposed NS framework.}

We split the whole dataset into seven parts for a comprehensive comparison. Except for the training set, we list the six test sets in order of classification difficulty from easy to hard:
\begin{itemize}
	\item {\bf Closed test set}: known devices without device aging, all conditions identical to the training set;
	\item {\bf Open 1}: unknown devices without device aging, all conditions identical to the training set;
	\item {\bf Open 2}: known devices with unknown device aging, collected 18 months later;
	\item {\bf Open 4}: unknown devices with unknown device aging, collected 18 months later;
	\item {\bf Open 2-3}: known devices with two types of unknown device aging, collected 18 months later;
	\item {\bf Open 4-5}: unknown devices with two types of unknown device aging, collected 18 months laters.
\end{itemize}
The Closed test set is the simplest known devices test set that with channel conditions similar to the training set. In contrast, Open 4-5 is the most challenging test set with unknown devices and different unknown device aging~(sampled at different dates).

\paragraph{Baselines and the Proposed NS-based RFF} We consider three types of baselines classified according to whether they are model-driven, purely data-driven, or mode-and-data-driven. As shown in Table \ref{tb:baseline}:
	\begin{itemize}
		\item {\bf M}: model-driven methods, which use handcrafted operations for the signal preprocessing, with different RFF extractors and auxiliary linear classifier back-ends.
		\item {\bf D}: data-driven methods that use only neural networks to extract the RFF.
		\item {\bf M\&D}: the proposed methods that combines signal processing priors with learnable parameters. 
	\end{itemize}
	
In Table~\ref{tb:baseline}, TS refers to the traditional synchronization, and NS indicates that the design uses the proposed neural synchronization. We also denote ``HP'' as the proposed hyperspherical projection. Except for the model in~\cite{yu2019multi} which has 63 million parameters, the number of parameters for the other methods is restricted to 12 million. \bflag{All models are trained using  the training data set shown in Fig. \ref{fig:dataset} for 150 epochs using the Adam optimizer~\cite{kingma2014adam} with a learning rate of $10^{-4}$ ($\beta_1=0.5, \beta_2=0.99$). The codes for reproducing our experiments are available at https://github.com/xrj-com/NS-RFF.}

\paragraph{Metric} 
Similar to biometric identification systems, we use the receiver operating characteristic~(ROC) curve, area under the ROC curve~(AUC), and the equal error rate~(EER) as the metrics to evaluate the quality of the extracted RFF. The ROC curve is obtained by plotting the true positive rate~(TPR) against the false positive rate (FPR) at various thresholds $T$~\cite{fawcett2006introduction}. Given the true positive~(TP), the true negative~(TN), the false positive~(FP), and the false negative~(FN) rates, TPR and FPR are respectively defined as:

\begin{equation}
  \text{TPR} = \frac{\text{TP}}{\text{TP} + \text{FN}},\quad \text{FPR} = \frac{\text{FN}}{\text{FP} + \text{TN}}.
\end{equation}

The TPR is also known as the probability of detection, defining how many correct positive samples~(intra-class) occur among all positive samples from the test. 
Here, positive samples refer to the signal pairs from the same device (similarly, negative samples refer to the pairs from different devices.)
FPR is also known as the probability of false alarm and is the percentage of correct negative samples~(inter-class) to the total negative samples. The ROC curve depicts the relative trade-offs between TPR and FPR. 
The EER refers to the point where FNR and FPR are equal; here FNR = 1-TPR. 

The higher AUC and the lower EER mean that the ROC curve is closer to the top left corner, which is recognized as ``perfect classification". This means that we simultaneously achieve both fewer false negatives and fewer false positives.
\begin{figure*}[t]
	\centering
	\subfigure[TS-based RFF.]{
        \centering
        \begin{minipage}[t]{0.332\linewidth}
        \centering
        \includegraphics[width=\linewidth]{\rootpath/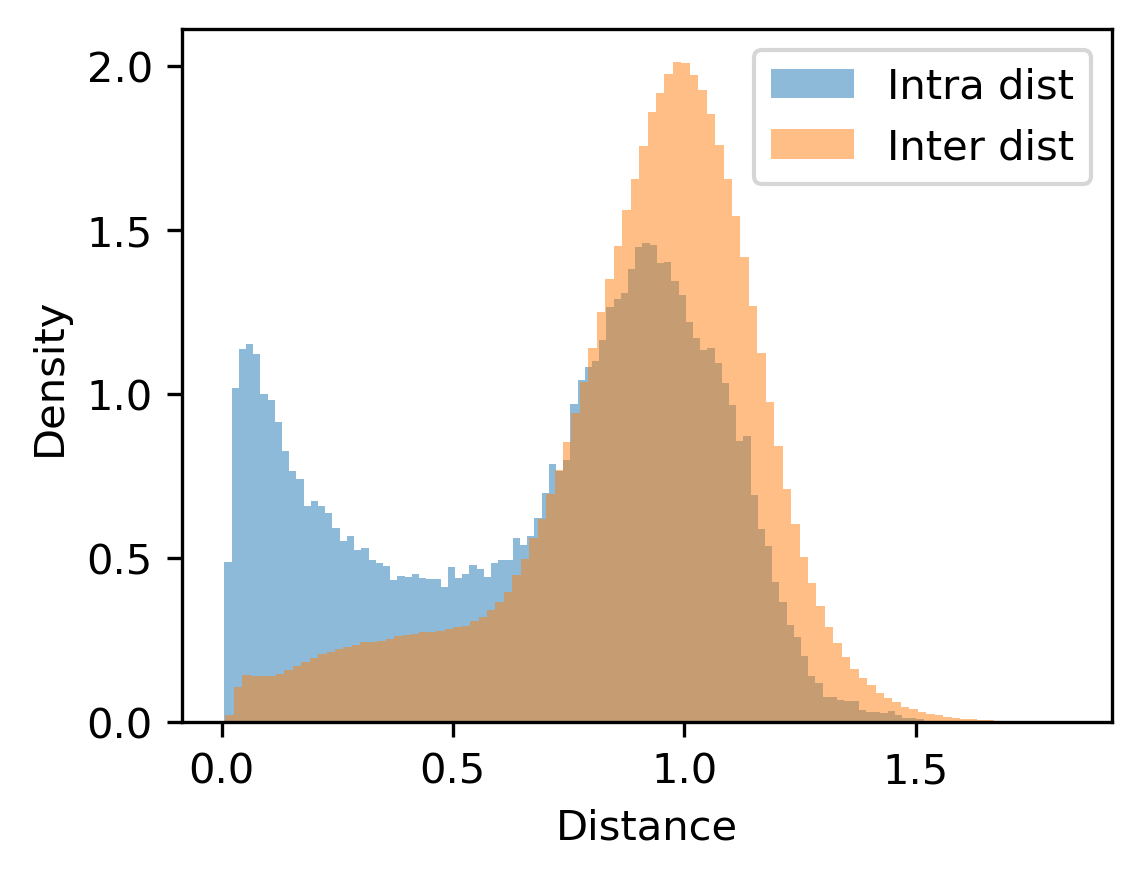}
        \end{minipage}
    }%
    \subfigure[DL-based RFF.]{
        \centering
        \begin{minipage}[t]{0.32\linewidth}
        \centering
        \includegraphics[width=\linewidth]{\rootpath/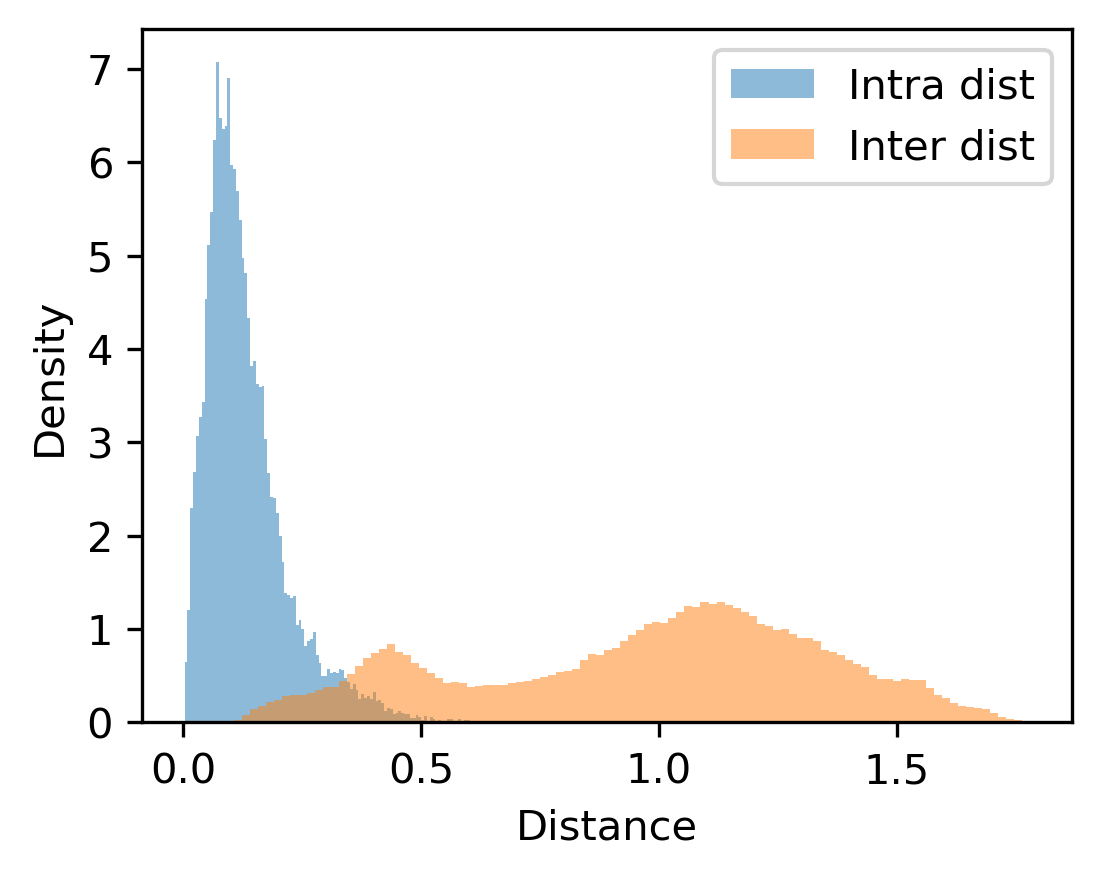}
        \end{minipage}
    }%
    \subfigure[NS-based RFF.]{
        \centering
        \begin{minipage}[t]{0.335\linewidth}
        \centering
        \includegraphics[width=\linewidth]{\rootpath/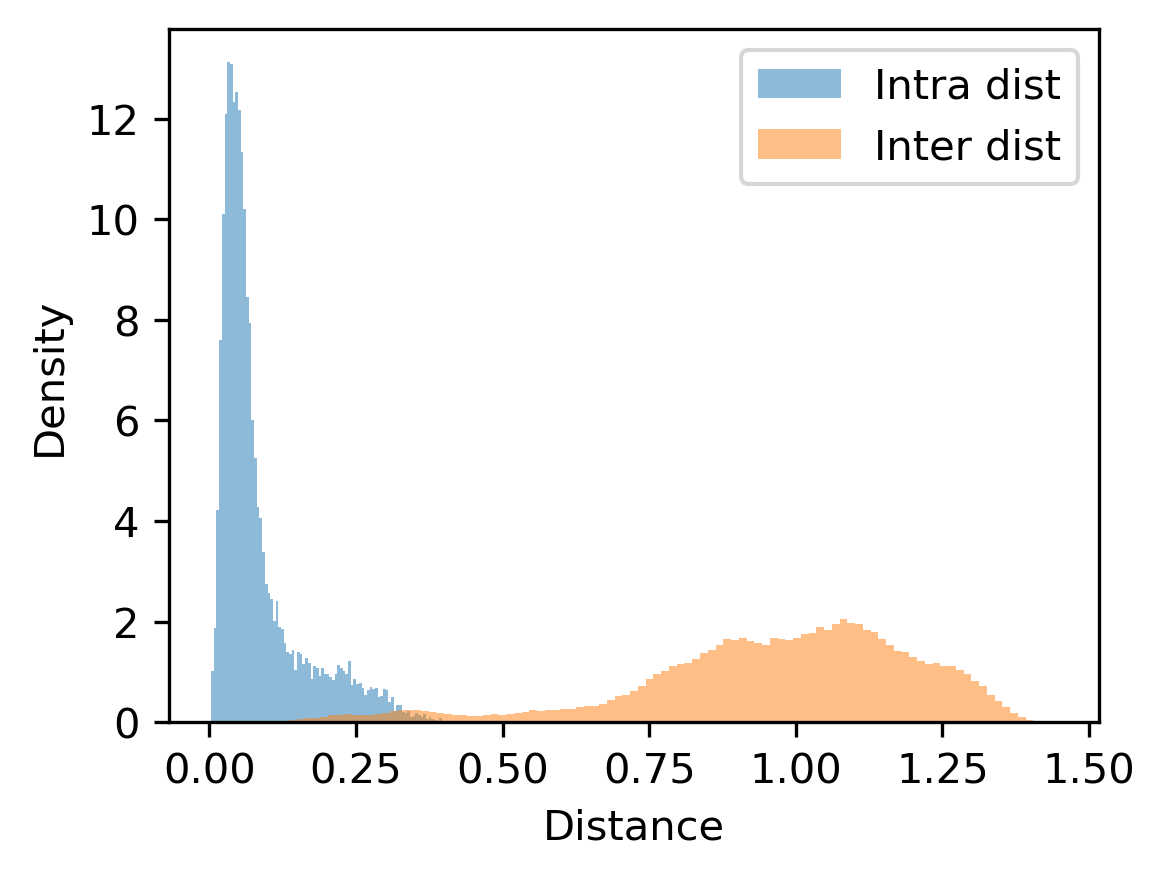}
        \end{minipage}
    }
	\caption{\bflag{Comparison via signal visualization (SNR = 30 dB). ``Intra dist'' indicates the distance distribution among the RFFs which came from the same device. Similarly, ``Inter dist'' is the distance distribution obtained from different devices.}}
	\label{fig:dist}
\end{figure*}

\input{tables/results}

\subsection{Performance Under Closed and Open Sets Settings}

To validate the superiority of the proposed model-and-data framework, we compare its performance with existing methods. For this purpose, we plot the ROC curves for the proposed method and the baselines for both closed-set and open-set settings, and we compare  the EER in Table~\ref{tb:resluts}. All results are measured on the test set. 

\paragraph{Power of End-to-End Learning} Overall, end-to-end methods (BCNN, NS+BCNN) outperform their non-end-to-end counterparts (Yu et.al\cite{yu2019robust}, TS+BCNN) by a large margin in all cases. These results verify our hypothesis: traditional preprocessing steps like carrier synchronization used in these methods does result in the loss of information about the device identities. We further discover that on cases with different sampled channels, non-end-to-end methods perform no better than random guesses (see e.g.~Fig.~\ref{fig:exp0}(b) to Fig.~\ref{fig:exp0}(d)). This suggests that these TS-based methods may indeed rely more on channel distinctions rather than hardware imperfection to distinguish devices. The use of deep neural networks, unfortunately, does not help to solve this problem. The result fully demonstrates the sub-optimality of TS for open-set RFF authentication and highlights the need for end-to-end learning.

\paragraph{Power of Signal Processing Priors} To investigate the usefulness of the proposed NS module, we compare the performance of the two end-to-end methods, namely BCNN (which directly learns the RFF from raw signals) and BCNN + NS (which preprocesses the signal with the NS module). We discover that BCNN with NS significantly outperforms BCNN without NS, as evidenced by both the ROC curve in Fig.~\ref{fig:exp0} and the EER values in Table II. This is especially the case for highly open-set scenarios (e.g., for unknown device+device aging, as shown in Fig.~\ref{fig:exp0}(e)(f)) where the difference between the ROC curves of the two methods are enlarged. This confirms the advantage of the inductive bias brought by the proposed NS module and demonstrates its necessity for open-set RFF authentication.

\paragraph{Power of Hypersphere Representation} In this section we further investigate the usefulness of the proposed hypersphere representation. We do this by applying the HP operation to \emph{all} methods. The results are presented in Table II. Interestingly, we find that while HP significantly improves the performance of the proposed model-and-data driven approaches, it deteriorates the performance of both model- and data-driven methods. We conjecture the underlying reason may be that HP is designed to encourage the separation between different RFFs in \emph{training set}, and when the RFFs learned from the training set can not generalized to the test set, this encouragement instead causes overfitting. In other words, only those methods that can extract highly generalizable RFFs will have good affinity for HP. This interesting result also serves as indirect evidence that the signal processing priors used in the NS module indeed helps for learning high-quality RFFs that are better generalized to unseen data.

\begin{figure*}[t]
	\centering
	\subfigure[Closed test set: Known devices/no device aging.]{
        \centering
        \begin{minipage}[t]{0.33\linewidth}
        \centering
        \includegraphics[width=\linewidth]{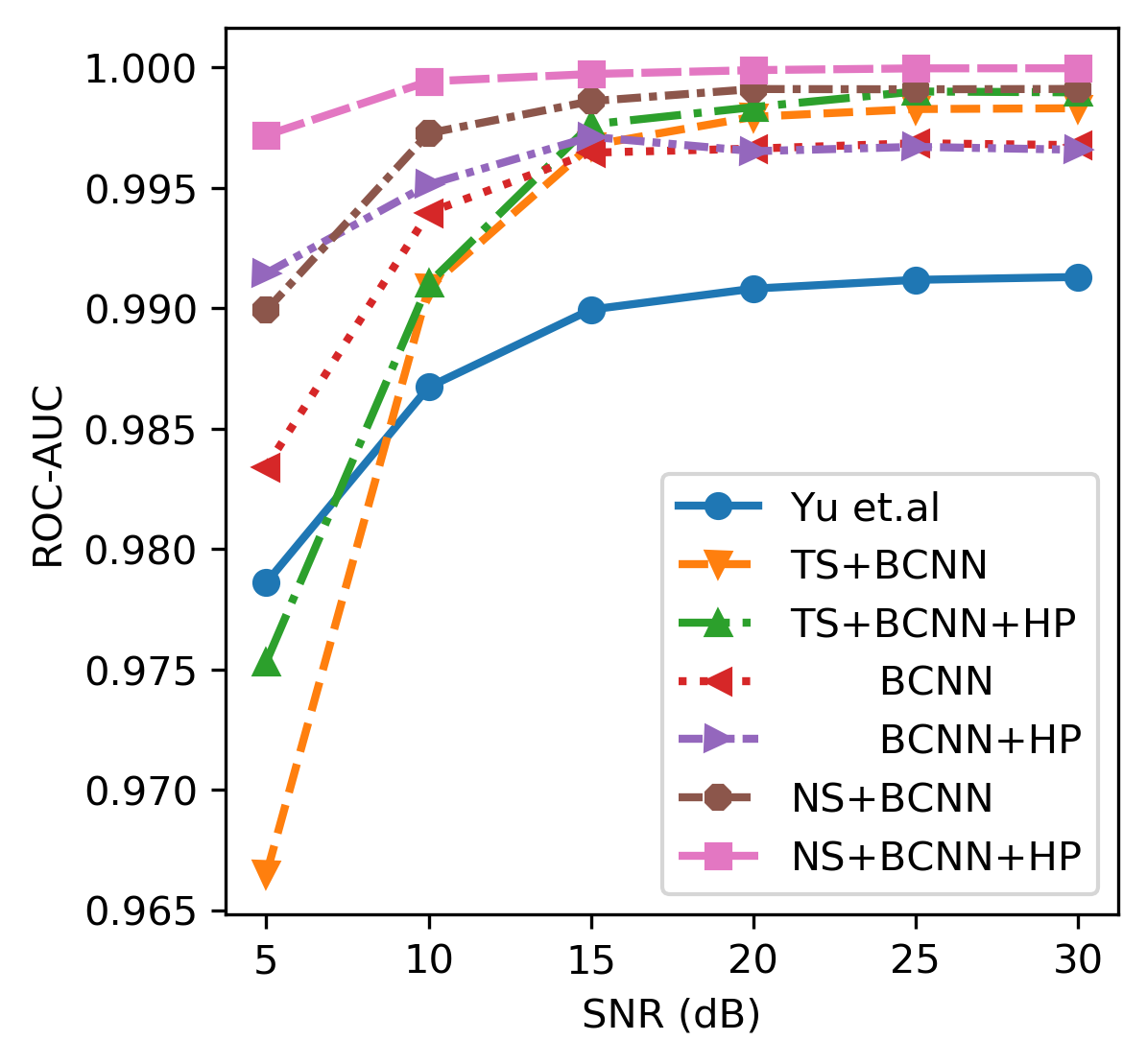}
        \end{minipage}
    }%
    \subfigure[Open 2: Known devices/device aging.]{
        \centering
        \begin{minipage}[t]{0.32\linewidth}
        \centering
        \includegraphics[width=\linewidth]{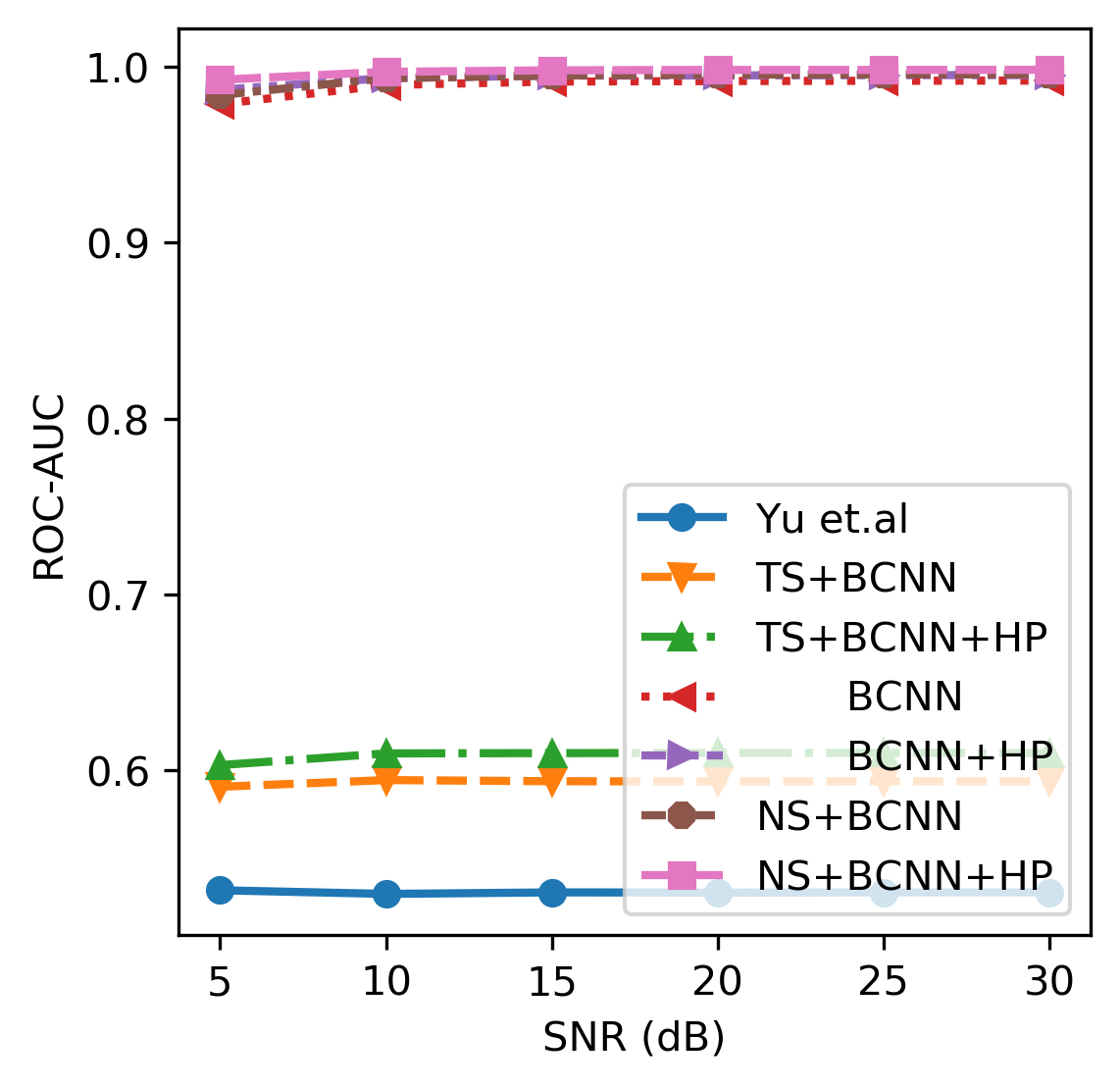}
        \end{minipage}
    }%
    \subfigure[Open 2-3: Known devices/device aging.]{
        \centering
        \begin{minipage}[t]{0.32\linewidth}
        \centering
        \includegraphics[width=\linewidth]{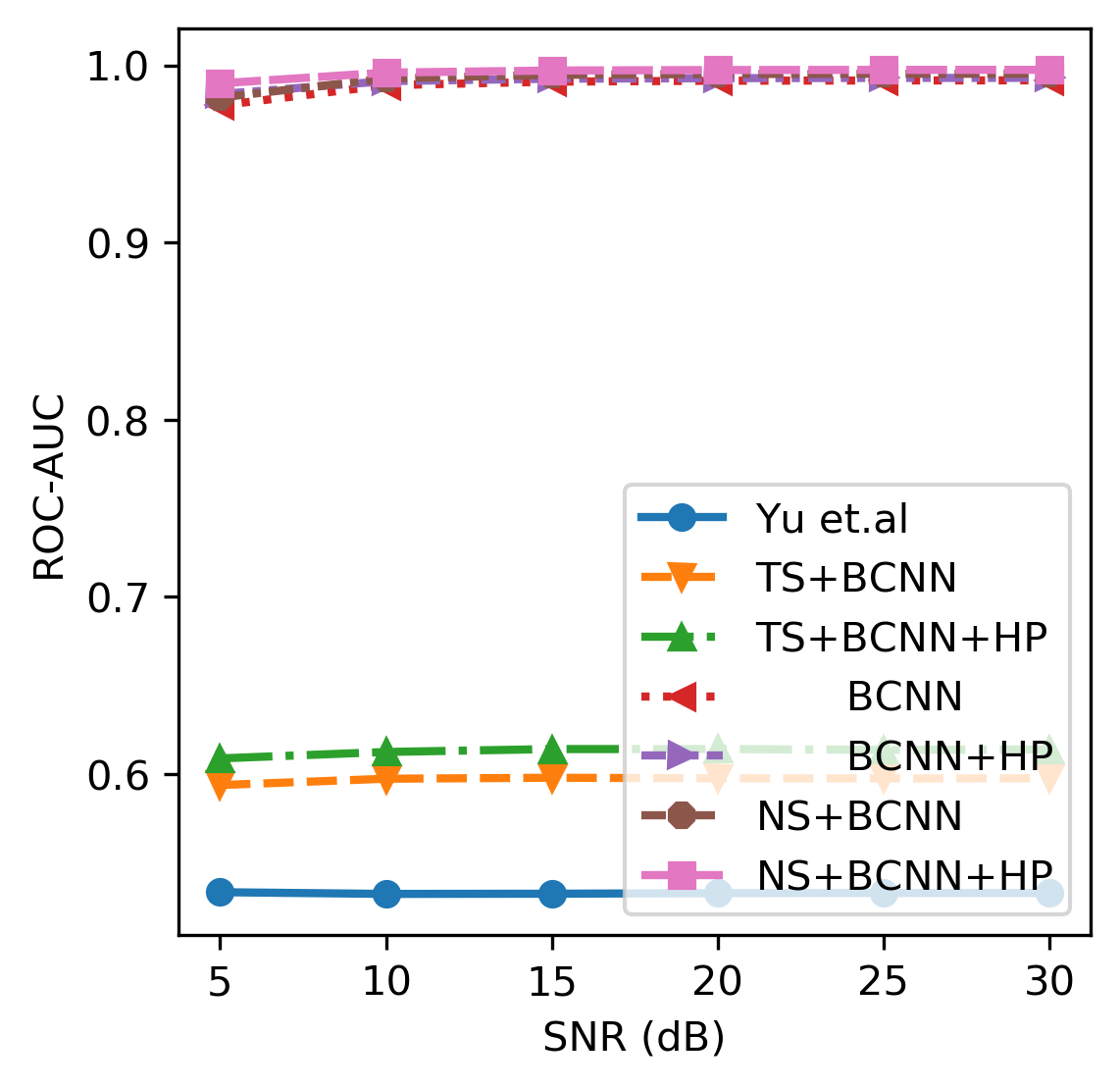}
        \end{minipage}
    }\\
    \subfigure[Open 1: Unknown devices/no device aging.]{
        \centering
        \begin{minipage}[t]{0.33\linewidth}
        \centering
        \includegraphics[width=\linewidth]{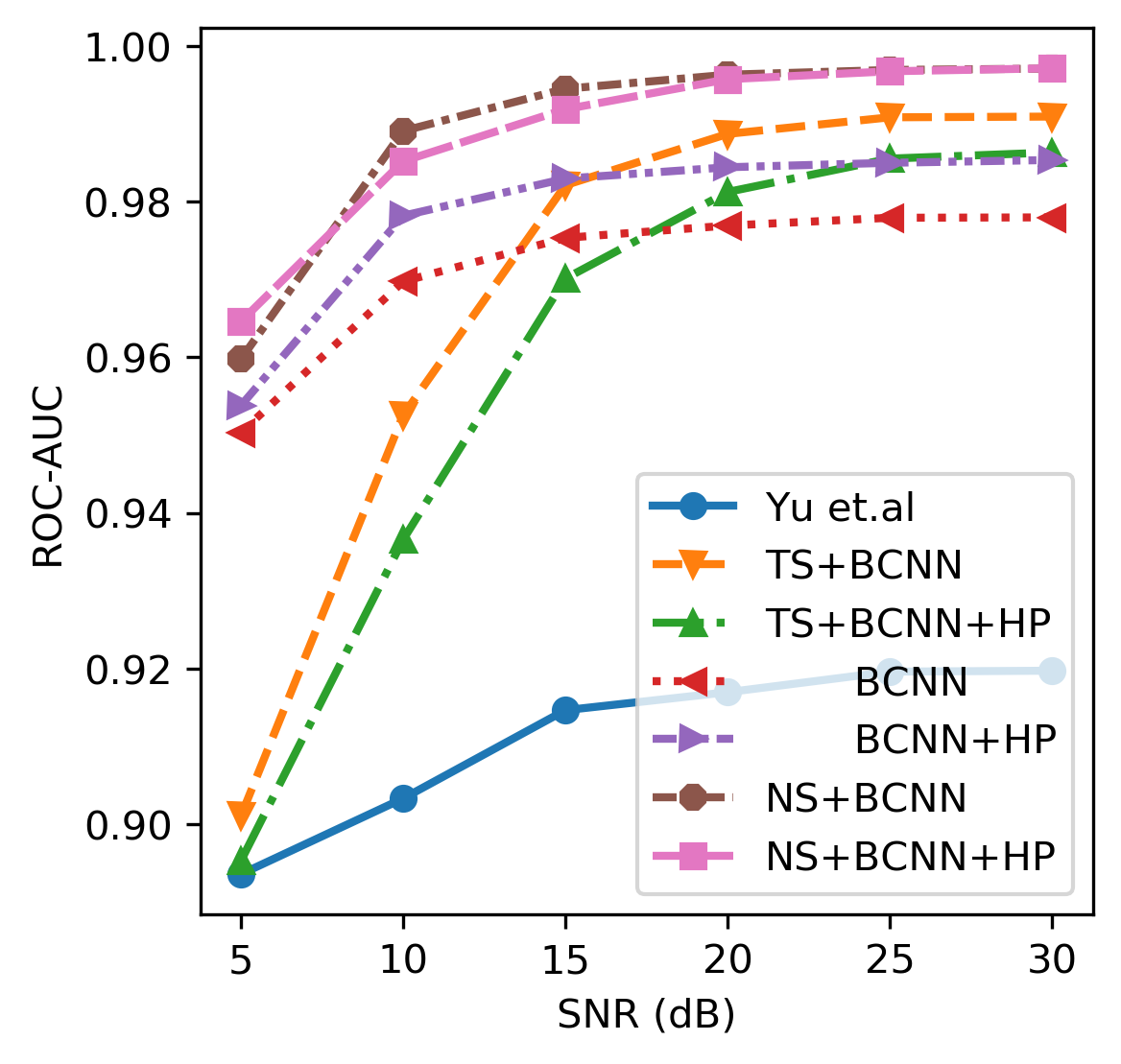}
        \end{minipage}
    }%
    \subfigure[Open 4: Unknown devices/device aging.]{
        \centering
        \begin{minipage}[t]{0.32\linewidth}
        \centering
        \includegraphics[width=\linewidth]{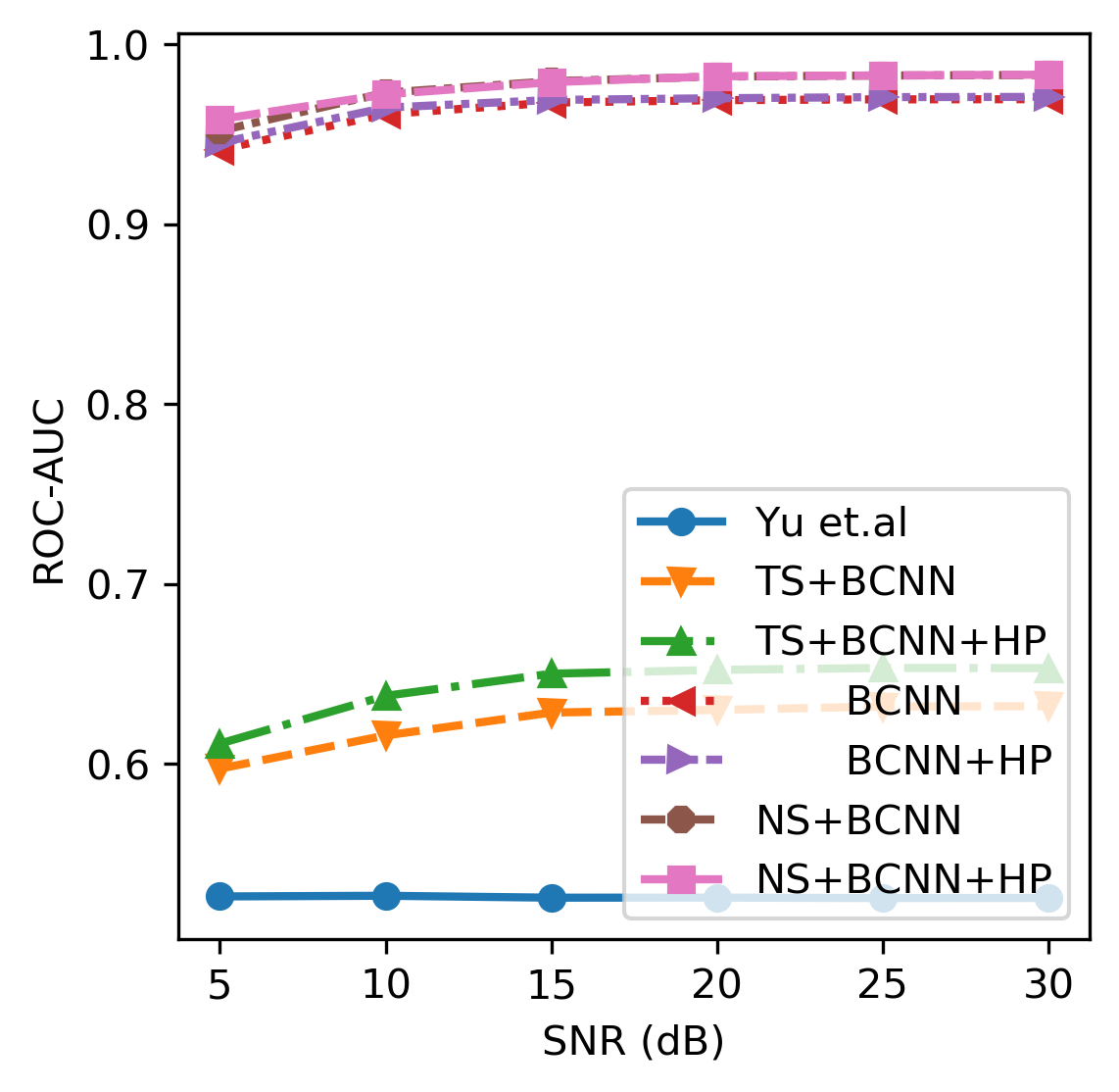}
        \end{minipage}
    }%
    \subfigure[Open 4-5: Unknown devices/device aging.]{
        \centering
        \begin{minipage}[t]{0.32\linewidth}
        \centering
        \includegraphics[width=\linewidth]{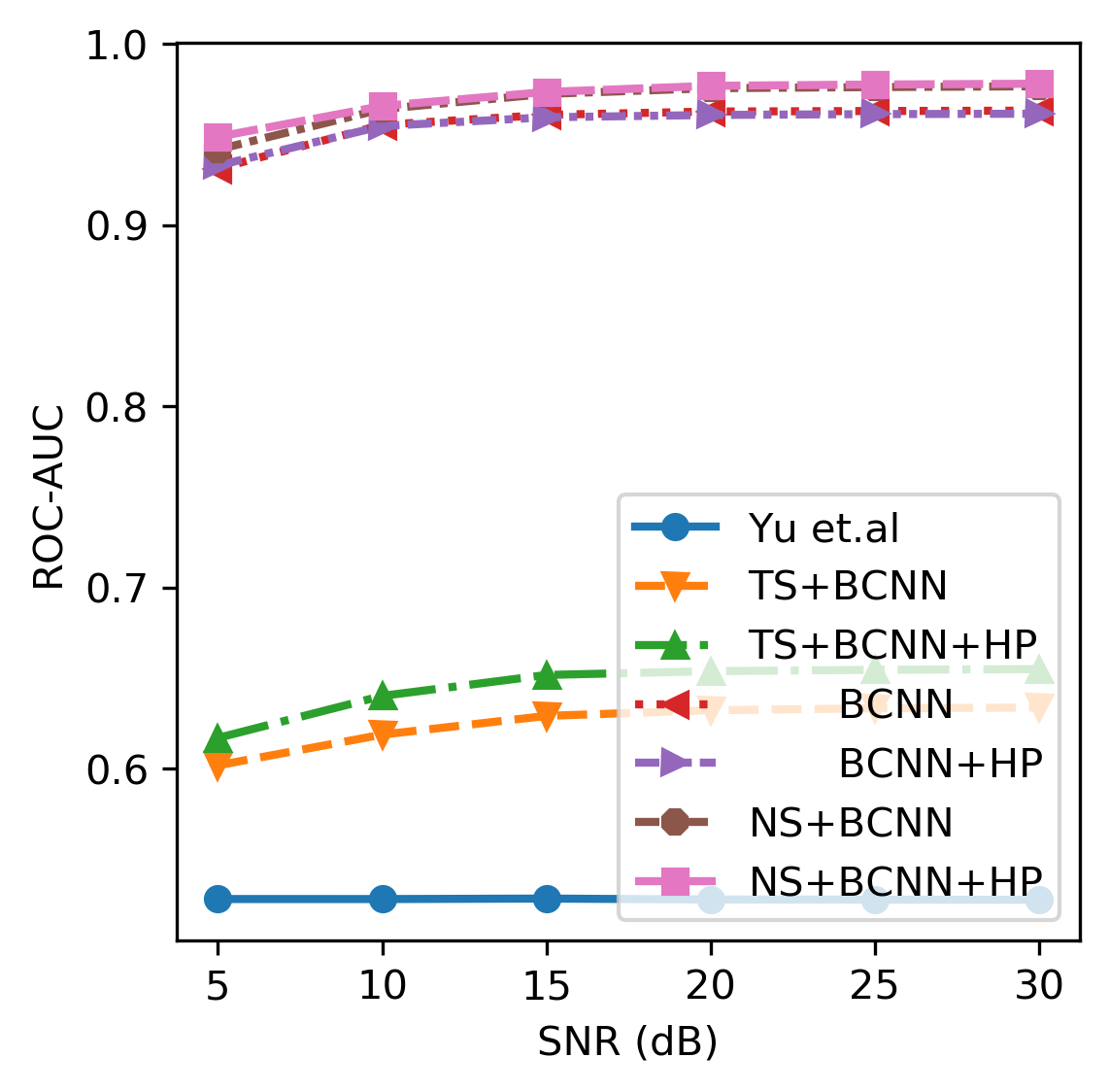}
        \end{minipage}
    }\\
	\caption{AUC-SNR curve of different methods under close and open test sets. }
	\label{fig:exp1}
\end{figure*}

\subsection{Visualizations}

\paragraph{Visualization of the Distance Distribution} To better provide intuition about the superiority of the proposed method, we present histograms of the distances between intra- and inter-device RFFs in Fig.~\ref{fig:dist}. These distances are calculated on the Open 4 dataset, which is a mixture of unknown and aged devices. It can be seen from Fig.~\ref{fig:dist} that intra-device and inter-device distances are nearly indistinguishable from each other for the TS-based RFF, which explains its poor performance on this dataset. Although the situation is much better in DL-based RFF, the overlap between intra- and inter-device distances is still evident in this method. Unlike these two methods, the distributions of the intra-device and inter-device distances for the NS-based method are well-separated, meaning that one can easily distinguish one device from the others. This illustrates why the proposed framework can achieve satisfactory performance in even open-set settings. 

\paragraph{Visualization of the Frequency and Phase Offsets} \bflag{
Fig. \ref{fig:visoffset}(a) and Fig. \ref{fig:visoffset}(b) show the scatter plots comparing the offsets estimated by TS and the proposed NS, respectively. We use different colors to indicate the device identity. Each point in Fig.  \ref{fig:visoffset} represents the frequency and phase pairs, i.e., ($\omega_{\text{TS}}$, $\phi_{\text{TS}}$). }

\bflag{
In Fig. \ref{fig:visoffset}(a), the offsets from the same device tend to be clustered over a small frequency range and can be separated by a linear classifier. This implies that the offsets from TS have a certain degree of correlation with the device identities. In fact, some device-dependent information can be lost if $\omega_{\text{TS}}$ and $\phi_{\text{TS}}$ are estimated and then removed from the received signals before the RFF extraction. Therefore, the RFF discrimination performance is weakened by using TS. 
}
\bflag{
In contrast, as shown in Fig. \ref{fig:visoffset}(b), the ($\omega_{\text{NS}}$, $\phi_{\text{NS}}$) from the proposed NS are randomly distributed on the plane. This means that the estimates $\omega_{\text{NS}}$ and $\phi_{\text{NS}}$ obtained by the proposed NS approach exhibit little dependence on the device identity. Thus, this suggests that NS removes device-irrelevant information from the input signals, and device-relevant information is better retained for subsequent RFF extraction. }

\subsection{Performance Versus SNR}

In this subsection, we further investigate the robustness of the proposed method with respect to noise. For this purpose, we artificially add random noise with different signal-to-noise ratio (SNR) to the input signal and investigate how the performance varies with SNR. Noise levels of SNR=$\{5,10,15,20,25,30\}$~dB are considered. \bflag{We also retrain all the models by data augmentation with random SNRs from 5 to 30 dB.}
The results are presented in Fig.~\ref{fig:exp1}. 

Again, we see that end-to-end learning methods are much more robust to noise than preprocessing-based methods even in closed-set and weakly open-set (i.e.~cases without  device aging) settings. The reason why traditional preprocessing methods are not robust may be due to their inability to achieve high mutual information $I(F(\vecr); \vecy)$ (here $F$ is the routine for computing the RFF), as we previously analyzed. This low mutual information will cause the system to be sensitive to injected noise --- an issue familiar to the communication community. End-to-end learning methods, by contrast, can better attain high mutual information between the extracted RFF and device identity, and hence are much less sensitive to noise.   

It is worth highlighing that the performance drop due to injected noise in TS-based methods is not comparable to that due to a changed channel (i.e.,~device aging). This again confirms that TS-based methods tend to overfit due to variations in the channels.

\begin{figure*}[t]
	\centering
	\subfigure[Closed test set: Known devices/no device aging.]{
        \centering
        \begin{minipage}[t]{0.32\linewidth}
        \centering
        \includegraphics[width=\linewidth]{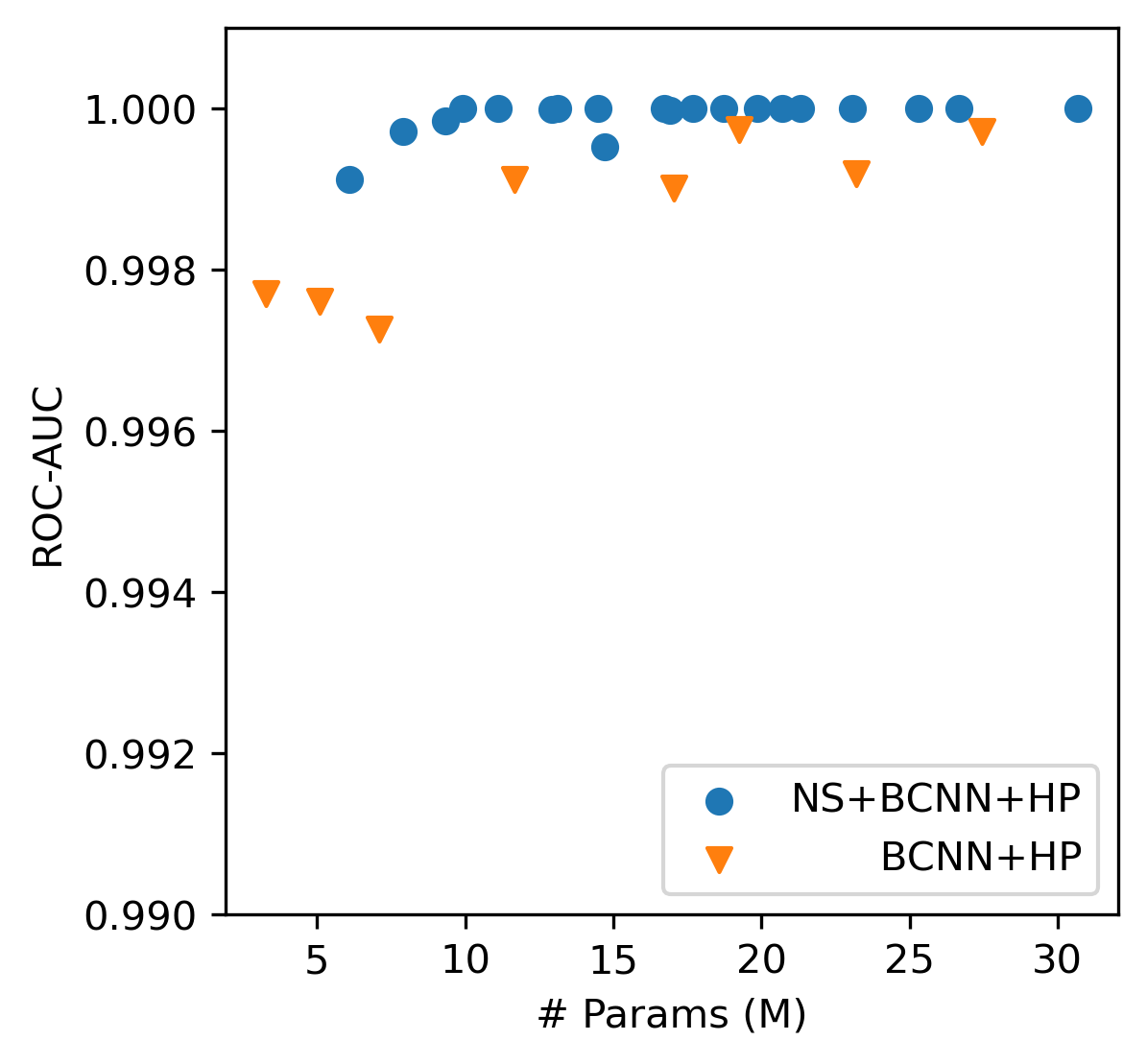}
        \end{minipage}
    }%
    \subfigure[Open 2: Known devices/device aging.]{
        \centering
        \begin{minipage}[t]{0.32\linewidth}
        \centering
        \includegraphics[width=\linewidth]{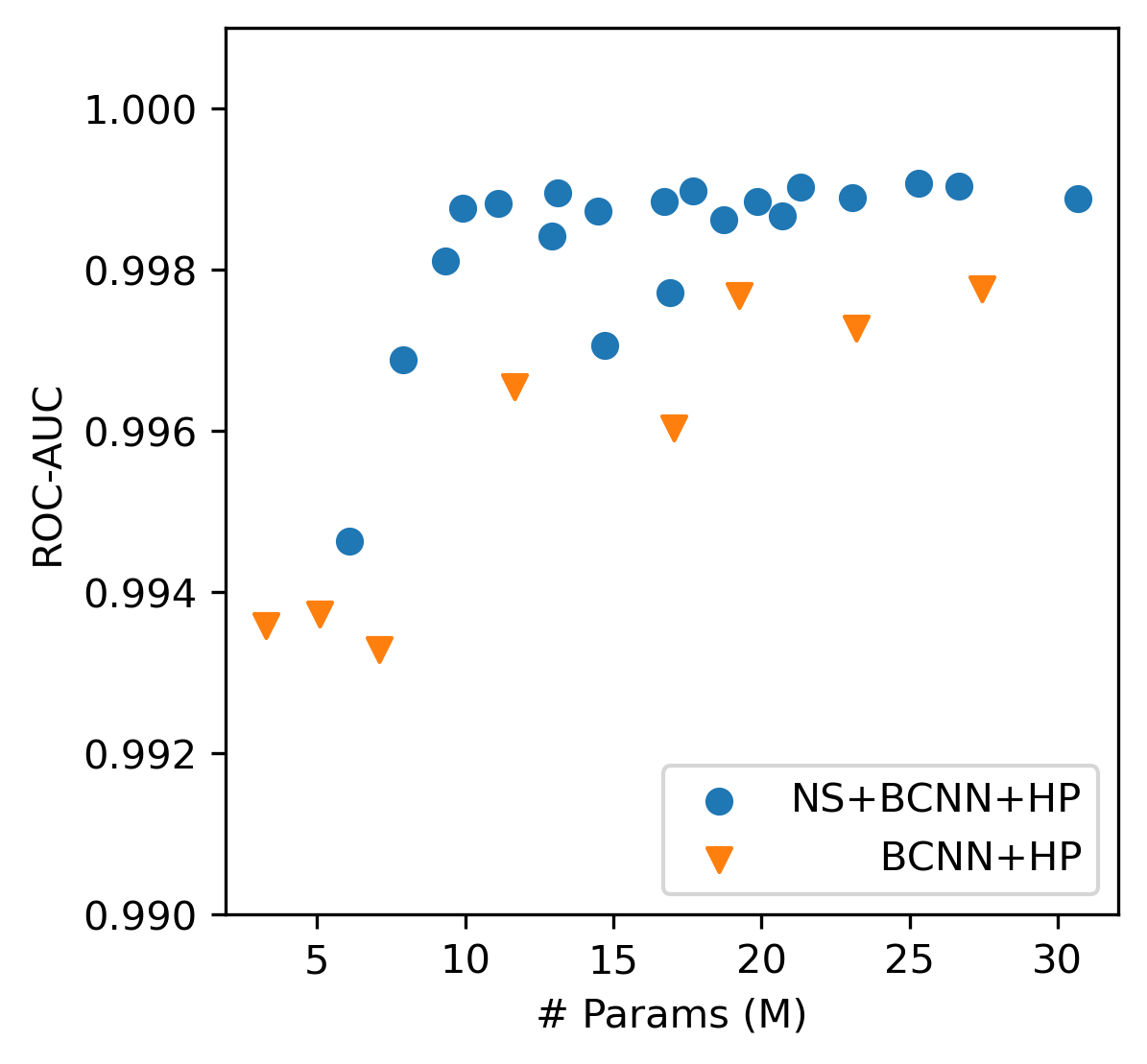}
        \end{minipage}
    }%
    \subfigure[Open 2-3: Known devices/device aging.]{
        \centering
        \begin{minipage}[t]{0.32\linewidth}
        \centering
        \includegraphics[width=\linewidth]{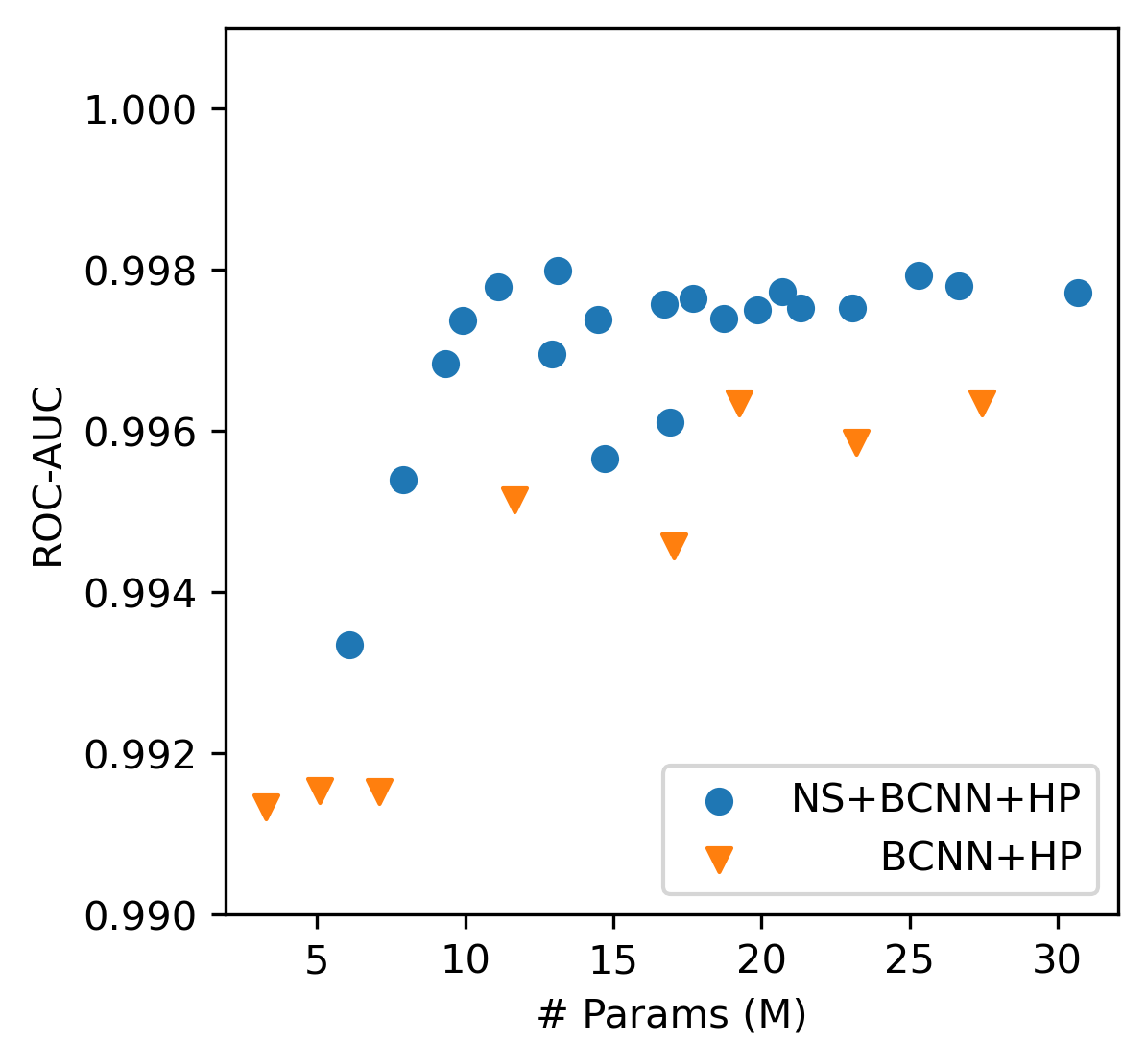}
        \end{minipage}
    }\\
    \subfigure[Open 1: Unknown devices/no device aging.]{
        \centering
        \begin{minipage}[t]{0.32\linewidth}
        \centering
        \includegraphics[width=\linewidth]{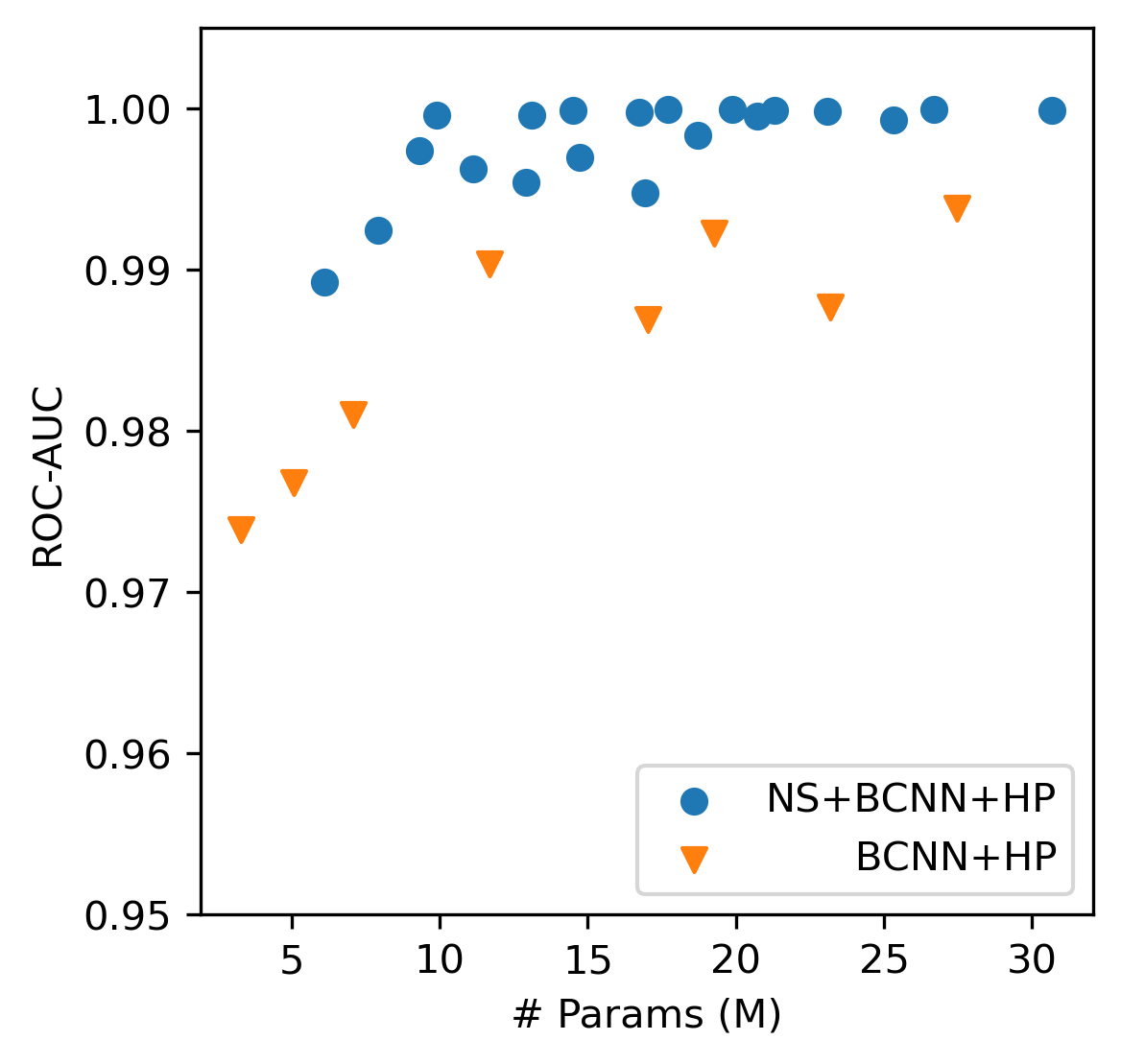}
        \end{minipage}
    }%
    \subfigure[Open 4: Unknown devices/device aging.]{
        \centering
        \begin{minipage}[t]{0.32\linewidth}
        \centering
        \includegraphics[width=\linewidth]{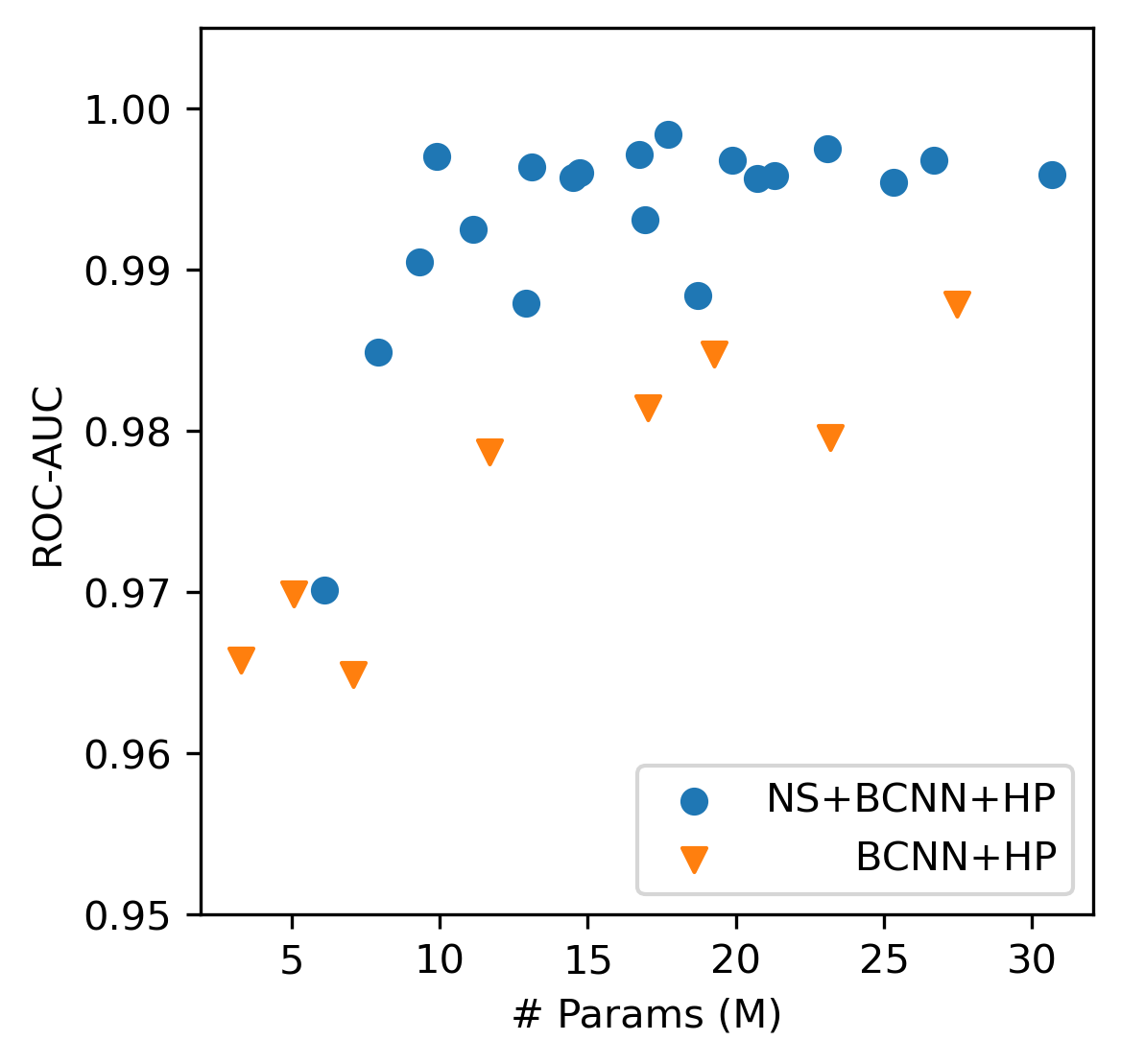}
        \end{minipage}
    }%
    \subfigure[Open 4-5: Unknown devices/device aging.]{
        \centering
        \begin{minipage}[t]{0.32\linewidth}
        \centering
        \includegraphics[width=\linewidth]{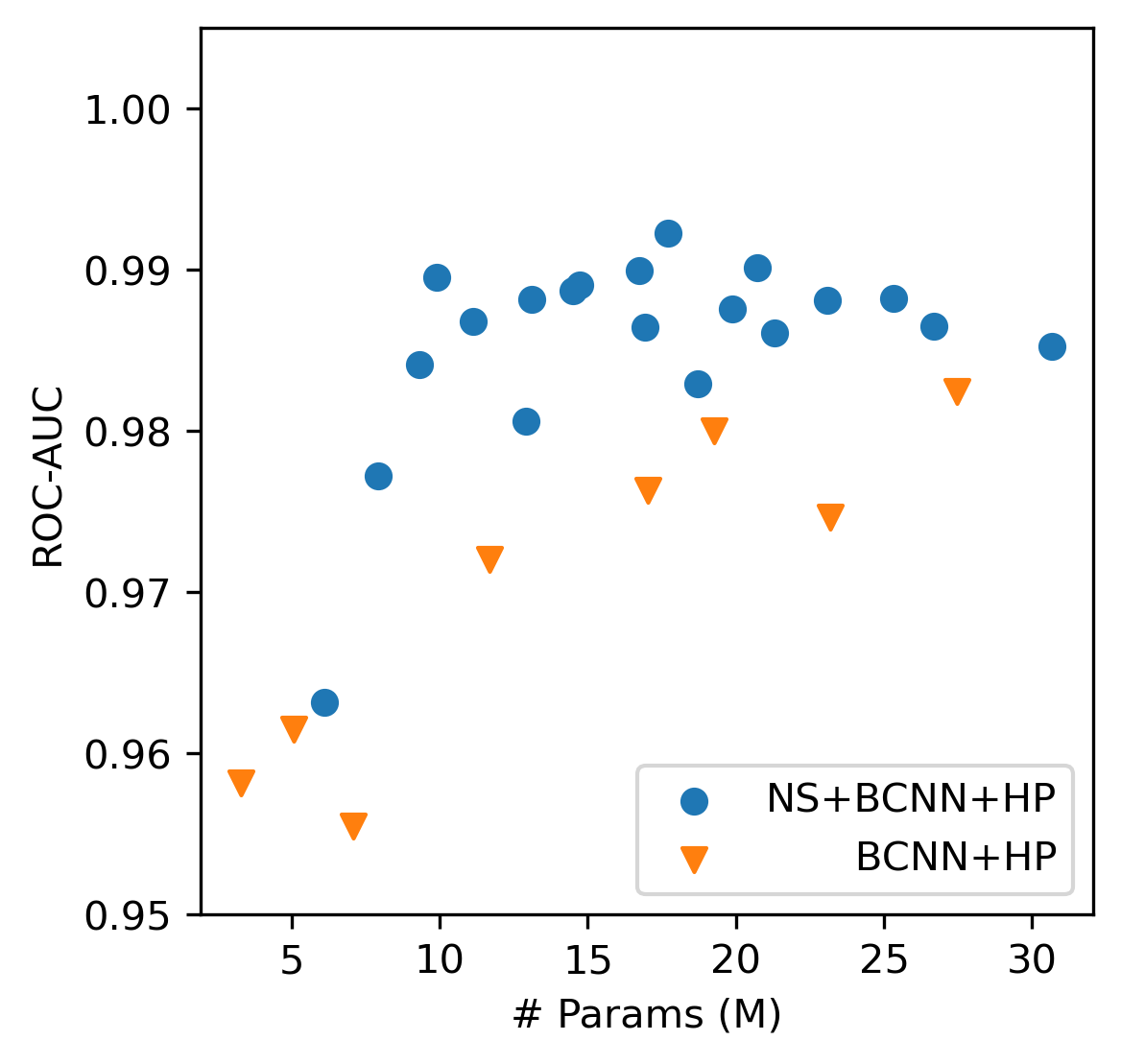}
        \end{minipage}
    }\\
	\caption{Scatter plots of \#Param-AUC under closed and open test sets.}
	\label{fig:exp2}
\end{figure*}

\input{tables/params.tex}

\subsection{Comparison of Varied Complexity}
In this subsection, we conduct a final experiment that studies the effects of the model complexity and the effectiveness of NS. We construct the model architectures by traversing the combination of values for $L_{\text{NS}}$ and $L_{\text{RFF}}$ in Table~\ref{tb:params}. 
\bflag{Here, we use $L_{\text{NS}}$ and $L_{\text{RFF}}$ to control the number of filters of the neural networks in the NS and the RFF extractor, respectively.}

The resulting scatter plots are shown in Fig.~\ref{fig:exp2}. The blue circles correspond to the proposed NS-based RFF with different levels of complexity, while the orange triangles correspond to the purely DL-based RFFs.

As seen in Fig.~\ref{fig:exp2}, for each test set, the typical BCNN without NS must become more complex to achieve satisfactory performance, while the convergence of the NS-based methods can easily achieve with only 9.9 \bflag{M} parameters~($L_{\text{NS}}=4$, $L_{\text{RFF}}=16$). 

We argue that this is because of the signal preprocessing priors to synchronization. The synchronization preprocessing restricts the neural network's model space while significantly reducing the difficulty of the learning task. This is also the reason why the proposed NS-based RFF has better performance in more general situations.

\begin{figure*}[t]
	\centering
    \subfigure[$\omega_{\text{TS}}$, and $\phi_{\text{TS}}$ from TS.]{
        \centering
        \begin{minipage}[t]{0.48\linewidth}
        \centering
        \includegraphics[width=\linewidth]{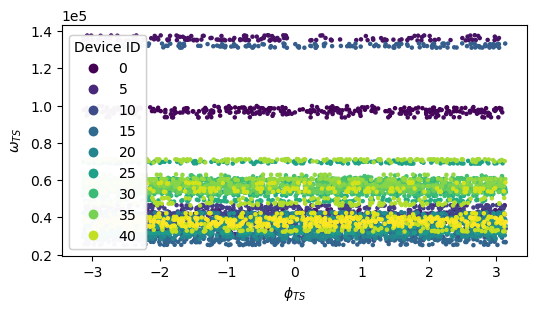}
        \end{minipage}
    }%
    \subfigure[$\omega_{\text{NS}}$, and $\phi_{\text{NS}}$ from NS.]{
        \centering
        \begin{minipage}[t]{0.48\linewidth}
        \centering
        \includegraphics[width=\linewidth]{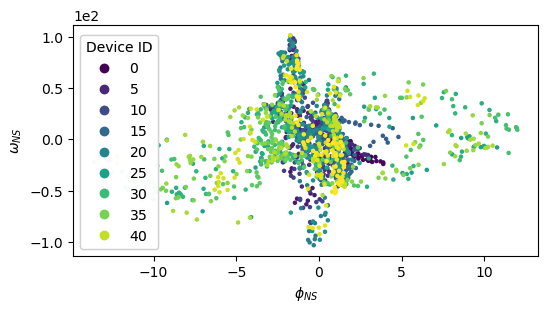}
        \end{minipage}
    }
	\caption{\bflag{Visualization of the frequency and phase offsets.}}
	\label{fig:visoffset}
\end{figure*}

\input{tables/classifiers}

\subsection{Performance of TS with Offsets as Additional Features}
\bflag{One may wonder whether the gain of the proposed NS-based RFF is just due to the gain derived from considering the frequency and phase offset in TS. To answer this question, we train a series of classifiers by taking corresponding sets as the training and evaluation data set. They include the following four different setups:}
\begin{itemize}
	\item \bflag{($\omega_{\text{TS}}$, $\phi_{\text{TS}}$): A classifier using only the frequency and phase offsets extracted by TS;}
	\item \bflag{$\mathbf{z}_{\text{TS}}$: A classifier using only RFFs extracted from the signal compensated by TS;}
	\item \bflag{$\mathbf{z}_{\text{TS}}$ + ($\omega_{\text{TS}}$, $\phi_{\text{TS}}$): A classifier using both the TS-based RFFs and the TS-based offsets;}
	\item \bflag{$\mathbf{z}_{\text{NS}}$: A classifier using RFFs extracted from the signal compensated by the proposed NS.}
\end{itemize}

\bflag{
Overall, we found that the performance of TS-based methods are not comparable with that of the NS-based method, even if we include offsets as additional features in these methods. In particular, the performance gain achieved by using ($\omega_{\text{TS}}$ , $\phi_{\text{TS}}$) is minor, only around $5\%$. In comparison, the performance gain by the proposed NS is significantly higher, around $28\%$. This indicates that ($\omega_{\text{TS}}$, $\phi_{\text{TS}}$) contains little additional information about the device identity, and the proposed NS does a much better job in preserving device-related information. 
}

\bflag{
One of the underlying reasons why ($\omega_{\text{TS}}$, $\phi_{\text{TS}}$) are less informative about the device identity is the imprecise estimation of $\omega_{\text{TS}}$ and $\phi_{\text{TS}}$. Note that TS estimates the frequency and phase offsets by comparing the received signal with an ideal signal. This means that the noise corrupting the received signal will also affect the precision of TS. When synchronizing the received signal by such imprecise $\omega_{\text{TS}}$ and $\phi_{\text{TS}}$, estimation errors are introduced in the processed signal, and this results in the loss of device-relevant information.
}

%% file: tables/baseline.tex
\begin{table*}[t]  

\caption{Baseline RFFs and Proposed NS-based RFFs. }

\centering
\begin{threeparttable}[b]
\begin{tabular}{c|c|c|c|c}  

\toprule 
\multirow{2}*{NN Design Types} &\multicolumn{3}{c|}{Methods~(Fig.~\ref{fig:diagram})}&\multirow{2}*{\#Parameters}\\
\cline{2-4}
 ~ & Preprocessing & RFF network & Auxiliary Classifier &~\\  
\midrule
 \multirow{3}*{Model-driven~(M)} & \multirow{3}*{TS} & Yu \emph{et.al}~\cite{yu2019multi} & Softmax  & 63 \bflag{M}\\
 \cline{3-5}
 ~ & ~ & \multirow{2}*{BCNN\tnote{\dag} } & Softmax & 12 \bflag{M}\\
 ~ & ~ & ~ & Softmax with HP\tnote{\dag}~~($\alpha=10$) & 12 \bflag{M}\\
\midrule
 \multirow{2}*{Data-driven~(D)} & \multirow{2}*{N/A}  & \multirow{2}*{BCNN\tnote{\dag} } & Softmax & 12 \bflag{M}\\
 ~ & ~ & ~ & Softmax with HP\tnote{\dag}~~($\alpha=10$)  & 12 \bflag{M}\\ 
\midrule
 \multirow{2}*{Model \& Data-driven~(M\&D)} & \multirow{2}*{NS\tnote{\dag}} & \multirow{2}*{BCNN\tnote{\dag} } & Softmax & 12 \bflag{M}\\
 ~ & ~  & ~ & Softmax with HP\tnote{\dag}~~($\alpha=10$)  & 12 \bflag{M}\\
\bottomrule  
\end{tabular}

\begin{tablenotes}
     \item[\dag] Proposed in this paper.
\end{tablenotes}

\end{threeparttable}

\label{tb:baseline}
\end{table*}

%
%
%
%
%
%
%

%% file: tables/results.tex
\begin{table*}[htbp]  

\caption{ROC Comparison of Different Methods. Model-driven~(M), Data-driven RFF~(D), and Model\&data-Driven~(M\&D).}

\centering
\begin{threeparttable}[b]
\begin{tabular}{llccccccccccc}  

\toprule   
\multirow{2}*{Types} & \multirow{2}*{Methods} &\multicolumn{2}{c}{Open 2}&\multicolumn{2}{c}{Open 2-3} &\multicolumn{2}{c}{Open 4}&\multicolumn{2}{c}{Open 4-5}\\
		~ & ~& AUC & EER & AUC & EER & AUC & EER& AUC & EER \\
\midrule
M & Yu et.al &0.5345 & 0.4828& 0.5374& 0.4807 & 0.5234 & 0.4822 & 0.5259 & 0.4820\\
M & TS + BCNN\tnote{\dag} &0.6717 & 0.3818 & 0.6699& 0.3824& 0.6619 & 0.3911 & 0.6653 & 0.3890\\
M &  TS + BCNN\tnote{\dag}~~(HP\tnote{\dag} ) &0.6078 & 0.4215 & 0.6072 & 0.4204 & 0.6085 & 0.4172 & 0.6136 & 0.4150\\
\midrule
D &  BCNN\tnote{\dag} & 0.9837 & 0.0593 & 0.9837& 0.0593 & 0.9669 & 0.0794 & 0.9590 & 0.0986\\
D &  BCNN\tnote{\dag}~~(HP\tnote{\dag} ) & 0.9933 & 0.0329 & 0.9915 & 0.0376 & 0.9649 & 0.0850 & 0.9555 & 0.1068\\ 
\midrule
M\&D &  NS\tnote{\dag}~~+ BCNN\tnote{\dag} & 0.9912 & 0.0400 & 0.9908 & 0.0418 & 0.9916 & 0.0487 & 0.9808 & 0.0739\\
M\&D & NS\tnote{\dag}~~+ BCNN\tnote{\dag}~~(HP\tnote{\dag} ) & {\bf 0.9990} & {\bf 0.0120} &{\bf 0.9976}&{\bf  0.0212} &{\bf 0.9984 }& {\bf 0.0197}& {\bf 0.9923} &{\bf 0.0456}\\ 

\bottomrule  

\end{tabular}
\begin{tablenotes}
     \item[\dag] Proposed in this paper.
\end{tablenotes}
\end{threeparttable}
\label{tb:resluts}
\end{table*}

%% file: tables/params.tex
\begin{table}[t]  

\caption{Complexity Setting of Network Architectures}

\centering
\begin{tabular}{r|c|c}  
\toprule   
 Types & $L_{\text{NS}}$ & $L_{\text{RFF}}$\\  
\midrule
\multirow{2}*{BCNN+HP} & \multirow{2}*{N/A} & 8, 12, 16, 24, 32, \\
~ & ~ & 35, 40, 45\\
\midrule
NS+BCNN+HP & 4, 8, 12, 16 & 8, 12, 16, 24, 32\\
\bottomrule  

\end{tabular}
\label{tb:params}
\end{table}


%% file: tables/classifiers.tex
\begin{table}[htbp]  

\caption{Close-set classification performance of RFF/$\omega$/$\phi$}

\centering
\begin{threeparttable}[b]
\begin{tabular}{lrr}  

\toprule   
\multirow{2}*{Training data} &\multicolumn{2}{c}{ACC}\\
		~ & Close test set & Open 2-3 \\

\midrule

$\omega_{\text{TS}}$, $\phi_{\text{TS}}$ & 20.4433\% & 14.1125\% \\ 
$\mathbf{z}_{\text{TS}}$ & 65.0246\% & 34.3815\% \\
$\mathbf{z}_{\text{TS}} + (\omega_{\text{TS}},\phi_{\text{TS}})$ & 70.6897\% & 39.4393\% \\
\midrule
$\mathbf{z}_{\text{NS}}$\tnote{\dag} & {\bf 98.7684\%} & {\bf 63.9704\%} \\ 
\bottomrule  

\end{tabular}
\begin{tablenotes}
     \item[\dag] Proposed NS in this paper.
\end{tablenotes}
\end{threeparttable}
\label{tb:clf}
\end{table}

%% file: sections/06_discussion.tex
\section{Conclusion}
This paper has proposed a new model-and-data-driven framework for open-set PLA. Traditional preprocessing techniques like TS have been widely used for RFF extraction.  However, such preprocessing may cause a loss of information about the device identity, according to the data processing inequality. Based on this observation, in the proposed framework, we use a  ``neural''  generalization of the carrier synchronization as a preprocessing module, referred to as NS.  This module serves as an essential part of the proposed end-to-end deep learning framework for introducing the inductive bias from signal processing models. We also proposed a hyperspherical representation to further improve the quality of the RFF identification. Experimental results show that TS-based methods tend to extract the weak features, i.e., channel distinction, rather than device-inherent features.  On the other hand,  the proposed NS module and the hypersphere representation with the proposed end-to-end training framework can guarantee the least information corruption and reduce the difficulty of the RFF learning task. The resulting RFF can not only perform well in known devices but can also be generalized to unknown devices and unknown channels.

\bflag{Some challenging tasks remain to be solved in future work:
 1) 
 For deployment in real-world scenarios, the complexity of the RFF extractor can be further reduced by the recent model compression techniques~\cite{liu2018rethinking, tan2019efficientnet};
 2) the proposed scheme is only developed for single-input single-output (SISO) system. How to devise a multiple-input multiple-output (MIMO) version of NS is another interesting and challenging research topic.}
 

%% file: sections/07_appendix.tex
\section{Proof of Theorem 1}
Given the Markov chain in \eqref{eq:markov2}, let us begin with $I(\vecz; \vecy)$: 
\begin{equation}
\begin{aligned}
I(\vecz; \vecy) &= \iint  p(\vecy, \vecz)  \ln \frac{p(\vecy, \vecz)}{p(\vecy)p(\vecz)} \dif \vecy \dif \vecz\\
                &= \iint p(\vecy, \vecz)  \ln \frac{p(\vecy| \vecz)}{p(\vecy)} \dif \vecy \dif \vecz\\
                &= \iint p(\vecy, \vecz)  \ln {p(\vecy| \vecz)} \dif \vecy \dif \vecz + \mathcal{H}(\vecy),
\end{aligned}
\label{eq:proof1}
\end{equation}
where 
\[
\mathcal{H}{(\vecy)}=\iint p(\vecy, \vecz)  \ln {p(\vecy)} \dif \vecy \dif \vecz=\int p(\vecy)  \ln {p(\vecy)} \dif \vecy
\]
denotes the entropy of $\vecy$, which is a constant and does not affect the optimization. The density $p(\vecy| \vecz)$ in (\ref{eq:proof1}) is fully defined by the proposed RFF extractor $p(\vecz| \vecr)$ and the given Markov chain as
\begin{equation}
\begin{aligned}
p(\vecy| \vecz) = \frac{p(\vecz , \vecy)}{p(\vecz)}
                = \int \frac{p(\vecz | \vecr)p(\vecy |\vecr)p(\vecr)}{p(\vecz)} \dif \vecr.
\end{aligned}
\label{eq:proof2}
\end{equation}
However, computing (\ref{eq:proof2}) is intractable. In order to accurately estimate $I(\vecz; \vecy)$, we use an auxiliary classifier $q(\vecy| \vecz)$ as a variational approximation of $p(\vecy| \vecz)$. Given the Kullback-Leibler divergence between $p(\vecy| \vecz)$ and $q(\vecy| \vecz)$, it follows that
\begin{equation}
\begin{aligned}
  	\mathcal{D}_{\text{KL}}[p(\vecy|\vecz)\parallel q(\vecy|\vecz)] = \int & p(\vecy|\vecz)  \ln \frac{p(\vecy|\vecz)}{q(\vecy|\vecz)} \dif \vecy \ge 0\\ 
  	\Rightarrow \int   p(\vecy| \vecz)  \ln p(\vecy| \vecz) \dif \vecy &\ge \int p(\vecy| \vecz)  \ln q(\vecy| \vecz) \dif \vecy.
\end{aligned}
\label{eq:proof3}
\end{equation}
By substituting (\ref{eq:proof3}) into (\ref{eq:proof1}), we have 
\begin{equation}
\begin{aligned}
I(\vecz; \vecy) &\ge \iint p(\vecz)p(\vecy| \vecz)  \ln {q(\vecy| \vecz)} \dif \vecy \dif \vecz + \mathcal{H}(\vecy)\\
&= \iiint p(\vecr) p(\vecy| \vecr) p(\vecz | \vecr)  \ln q(\vecy | \vecz)\dif \vecr \dif \vecy \dif \vecz + \mathcal{H}(\vecy)\\
&= \iiint  p(\vecy, \vecr) p(\vecz | \vecr)  \ln q(\vecy | \vecz)\dif \vecr \dif \vecy \dif \vecz + \mathcal{H}(\vecy).
\end{aligned}
\label{eq:proof4}
\end{equation}

The second equality follows by introducing the variable $\vecr$, which is the received signal. From (\ref{eq:proof3}), it is readily to seen that the equality in (\ref{eq:proof4}) holds if and only if $\mathcal{D}_{\text{KL}}[p(\vecy|\vecz)\parallel q(\vecy|\vecz)]=0$ which implies $q(\vecy|\vecz)=p(\vecy|\vecz)$.